\documentclass{article}

\usepackage[preprint]{neurips_2026}
\usepackage{subcaption}
\usepackage{graphicx}
\usepackage{placeins}
\usepackage{algorithm}
\usepackage{algorithmic}
\usepackage[utf8]{inputenc}
\usepackage[T1]{fontenc}
\usepackage{hyperref}
\usepackage{url}
\usepackage{booktabs}
\usepackage{multirow}
\usepackage{amsfonts}
\usepackage{nicefrac}
\usepackage{microtype}
\usepackage{xcolor}
\usepackage{tikz}
\usetikzlibrary{arrows.meta,positioning,bending}
\usepackage{amsmath, amsthm, amssymb}
\usepackage{bm}
\usepackage{mathtools}

\newtheorem{theorem}{Theorem}

\newtheorem{proposition}[theorem]{Proposition}
\newtheorem{corollary}[theorem]{Corollary}
\newtheorem{assumption}{Assumption}

\title{Regularized Centered Emphatic Temporal Difference Learning}

\author{%
  Xingguo Chen\thanks{Correspondence: \texttt{chenxg@njupt.edu.cn}.}, Chaohui Wu, Jinguo Ye, Chao Li, Shangdong Yang\\
  Nanjing University of Posts and Telecommunications\\
  \texttt{chenxg@njupt.edu.cn}
  \And
  Guang Yang\\
  Nanjing University
  \And
  Tianyu Liang\\
  Microsoft\\
  \texttt{liangtianyu@microsoft.com}
  \And
  Wenhao Wang\\
  National University of Defense Technology
}

\begin{document}

\maketitle

\begin{abstract}
Off-policy temporal-difference (TD) learning with function approximation faces a structural tradeoff among stability, projection geometry, and variance control. Emphatic TD (ETD) improves the off-policy projection geometry through follow-on emphasis, but the follow-on trace can have high variance. We revisit this tradeoff through Bellman-error centering. Although centering naturally removes a common drift term from TD errors, we show that a naive centered emphatic extension introduces an auxiliary coupling that can destroy the positive-definiteness of the ETD key matrix. We propose \emph{Regularized Emphatic Temporal-Difference Learning} (RETD), which preserves the follow-on trace and regularizes only the auxiliary centering recursion, corresponding to lifting the lower-right block of the coupled key matrix from \(1\) to \(1+c\). We derive the RETD core matrix, prove convergence under a conservative sufficient regularization condition, and evaluate the method on diagnostic linear off-policy prediction tasks. The experiments show that RETD avoids the instability of naive centered emphatic learning, preserves favorable emphatic geometry, and exhibits a robust intermediate regime for the regularization parameter \(c\) across the diagnostics.
\end{abstract}
\section{Introduction}

Off-policy temporal-difference (TD) learning with function approximation is difficult because stability, approximation geometry, and variance control are tightly coupled. The classical deadly triad can make ordinary TD unstable under off-policy sampling \citep{sutton2016emphatic}. Gradient-based methods, including GTD \citep{sutton2008convergent}, GTD2, TDC \citep{sutton2009fast}, proximal TD variants \citep{liu2016proximal,liu2018proximal}, and TDRC \citep{ghiassian2020gradient}, address this stability problem by optimizing objectives associated with TD fixed points. Their guarantees are valuable, but the resulting fixed point is still tied to the ordinary TD projection geometry and can be poorly aligned with the target value function, with approximation error that may be arbitrarily worse than the best approximation \citep{kolter2011fixed}; stability alone therefore does not settle the approximation question.

Emphatic TD (ETD) takes a different route. By using a follow-on trace, ETD changes the projection geometry from the behavior-weighted TD geometry to an emphatically weighted geometry that reflects target-policy bootstrapping \citep{sutton2016emphatic,hallak2016generalized}. This can yield a more appropriate fixed point, especially on diagnostic off-policy examples such as Baird's counterexample. The cost is variance: the follow-on trace accumulates products of importance-sampling ratios and can be highly variable, or even have infinite variance \citep{sutton2016emphatic,zhang2022truncated}. Much of the ETD literature therefore improves practical behavior by controlling, approximating, truncating, learning, or otherwise modifying the emphatic trace itself, including generalized ETD \citep{hallak2016generalized}, PER-ETD and related corrections \citep{guan2022per}, TETD \citep{zhang2022truncated}, deep and learned-emphasis methods \citep{jiang2021emphatic,jiang2022learning}, and LC-ETD \citep{he2023loosely}; all intervene on the trace pathway in one form or another.

This paper asks a complementary question: can we improve the emphatic update structure while preserving the original follow-on-trace backbone? Our starting point is Bellman-error centering \citep{chen2026bellman}. Centering removes a common drift term from the TD error by replacing the raw error with a centered error. Applying this idea to ETD is natural: keep the emphatic weighting, but subtract an online estimate of the mean emphatic TD error. We call the resulting naive construction Centered Emphatic Temporal-Difference Learning (CETD). However, CETD reveals a structural failure. The auxiliary centering variable introduces a low-rank coupling into the joint parameter--auxiliary system, and this coupling can destroy the positive-definiteness property inherited from ETD. In Section~\ref{sec:cetd}, we show this explicitly through the CETD key matrix and a two-state counterexample.

We propose Regularized Emphatic Temporal-Difference Learning (RETD) as a minimal repair. RETD does not truncate, learn, or replace the follow-on trace. Instead, it regularizes only the auxiliary centering recursion: in the coupled key matrix, the lower-right entry changes from \(1\) to \(1+c\). When \(c=0\), RETD reduces to CETD; for moderate \(c>0\), the harmful auxiliary coupling is damped while centering remains active; as \(c\) becomes very large, the auxiliary variable is strongly suppressed and the method becomes increasingly ETD-like. Thus \(c\) is best viewed as a structural control knob between CETD-like and ETD-like regimes, not as a conventional penalty whose largest value is automatically best for either accuracy or stability.

The theory and experiments are organized around this matrix mechanism. We derive the RETD core matrix in Section~\ref{sec:retd} and prove convergence in Section~\ref{sec:theory} under standard Markovian stochastic-approximation assumptions together with a conservative sufficient regularization condition. This general condition is not intended as a tuning rule: in a concrete problem, the relevant threshold is governed by the actual coupled matrix, as illustrated by the two-state calculations. The experiments in Section~\ref{sec:experiments} are diagnostic rather than claims of broad benchmark dominance. Boyan chain tests whether RETD preserves favorable emphatic geometry, Baird's counterexample tests stability under severe off-policy bootstrapping, and a new two-state prediction task isolates the difference between CETD and RETD. We report curves, tail-average RMSE, divergence counts, and learning-rate and regularization scans against TD, GTD2, TDC, TDRC, ETD, TETD, and CETD where applicable.

The main contributions are as follows. First, we formulate CETD as the emphatic counterpart of Bellman-error centering and derive its coupled key matrix. Second, we show that naive centering can lose positive definiteness even in a simple two-state example, explaining why centering and emphatic weighting do not automatically combine safely. Third, we propose RETD, which restores the key-matrix structure by regularizing only the auxiliary centering recursion while preserving the original emphatic trace. Fourth, we derive the RETD core matrix and prove convergence under a conservative sufficient regularization condition. Fifth, we validate the predicted stability--geometry tradeoff in diagnostic linear off-policy prediction experiments using tail RMSE, divergence statistics, stepsize slices, and scans over the regularization parameter \(c\), rather than broad benchmark claims.
\section{Preliminaries}\label{sec:preliminaries}

We consider off-policy policy evaluation in a finite Markov decision process with state space \(\mathcal{S}\), action space \(\mathcal{A}\), transition kernel \(P(\cdot\mid s,a)\), expected reward function \(r(s,a)\), discount factor \(\gamma\in[0,1)\), target policy \(\pi\), and behavior policy \(\mu\). Under linear function approximation, \(v_{\bm{\theta}}(s)=\bm{\phi}(s)^\top\bm{\theta}\), with feature matrix \(\bm{\Phi}\). We write \(\bm{P}_\pi(s,s')=\sum_a\pi(a\mid s)P(s'\mid s,a)\), \(\bm{r}_\pi(s)=\sum_a\pi(a\mid s)r(s,a)\), \(\bm{d}_\mu\) for the stationary distribution of \(\mu\), and \(\bm{D}_\mu=\mathrm{diag}(\bm{d}_\mu)\). We assume that the behavior-induced chain is ergodic and that \(\bm{\Phi}\) has full column rank.

For any diagonal positive definite matrix \(\bm{M}\), let \(\Pi_{\bm{M}}=\bm{\Phi}(\bm{\Phi}^\top\bm{M}\bm{\Phi})^{-1}\bm{\Phi}^\top\bm{M}\) denote projection onto \(\mathrm{span}(\bm{\Phi})\) under the \(\bm{M}\)-weighted inner product. The Bellman operator is \(\mathcal{T}_\pi\bm{v}=\bm{r}_\pi+\gamma\bm{P}_\pi\bm{v}\), and \(\bm{v}_\pi=(\bm{I}-\gamma\bm{P}_\pi)^{-1}\bm{r}_\pi\). Since \(\bm{v}_\pi\notin\mathrm{span}(\bm{\Phi})\) in general, the best \(\bm{M}\)-weighted approximation is \(\Pi_{\bm{M}}\bm{v}_\pi\) \citep{schoknecht2002optimality,scherrer2010should}.

Let \((S_t,A_t,R_{t+1},S_{t+1})\) be generated by \(\mu\), define \(\rho_t=\pi(A_t\mid S_t)/\mu(A_t\mid S_t)\), and let \(\delta_t=R_{t+1}+\gamma\bm{\phi}(S_{t+1})^\top\bm{\theta}_t-\bm{\phi}(S_t)^\top\bm{\theta}_t\). Ordinary off-policy TD(0) updates \(\bm{\theta}_{t+1}=\bm{\theta}_t+\alpha_t\rho_t\delta_t\bm{\phi}(S_t)\). Its fixed point \(\bm{v}_{\mathrm{TD}}=\bm{\Phi}\bm{\theta}_{\mathrm{TD}}\) satisfies \(\bm{v}_{\mathrm{TD}}=\Pi_{\bm{D}_\mu}\mathcal{T}_\pi\bm{v}_{\mathrm{TD}}\), and can be poorly aligned with \(\bm{v}_\pi\) under policy mismatch \citep{kolter2011fixed}.

ETD introduces the follow-on trace and emphatic update
\begin{equation}
F_0=1,\qquad F_t=1+\gamma\rho_{t-1}F_{t-1},\qquad
\bm{\theta}_{t+1}=\bm{\theta}_t+\alpha_tF_t\rho_t\delta_t\bm{\phi}(S_t).
\label{etdf}
\end{equation}
Its expected weighting is \(\bm{f}=(\bm{I}-\gamma\bm{P}_\pi^\top)^{-1}\bm{d}_\mu\), \(\bm{F}=\mathrm{diag}(\bm{f})\), and the ETD fixed point \(\bm{v}_{\mathrm{ETD}}=\bm{\Phi}\bm{\theta}_{\mathrm{ETD}}\) satisfies \(\bm{v}_{\mathrm{ETD}}=\Pi_{\bm{F}}\mathcal{T}_\pi\bm{v}_{\mathrm{ETD}}\). The key difference is geometric: TD is tied to the behavior-weighted projection \(\Pi_{\bm{D}_\mu}\bm{v}_\pi\), whereas ETD is tied to the emphatically weighted projection \(\Pi_{\bm{F}}\bm{v}_\pi\). Because \(\bm{F}\) reflects both behavior-policy visitation and target-policy bootstrapping, it often induces a more appropriate approximation geometry; empirically, \citet{hallak2016generalized} showed that the ETD fixed point is often closer to its optimal projection than the TD fixed point.

This geometric gain comes with higher variance. From \eqref{etdf}, the follow-on trace accumulates products of importance-sampling ratios, so \(F_t\rho_t\delta_t\bm{\phi}(S_t)\) can be substantially noisier than the TD update. Most ETD modifications therefore reduce variance through truncation, clipping, restarts, or other changes to the emphatic trace. Bellman-error centering offers a complementary route: it targets the TD-error structure directly by replacing an error with its centered version. This motivates asking whether centering can be incorporated into ETD without disrupting that emphatic mechanism.
\section{Centered Emphatic Temporal-Difference Learning}\label{sec:cetd}

We now introduce Centered Emphatic Temporal-Difference Learning (CETD). 
We begin with a brief review of off-policy centered temporal-difference learning (CTD), which provides the centering idea underlying CETD. 
As discussed in the preliminaries, the goal of centering is to reduce the common drift contained in the TD error by subtracting an estimate of its mean, so that the update is driven by a centered error rather than the raw TD error.

This idea can be understood more systematically through Bellman error centering \citep{chen2026bellman}. 
Let the centering operator \(C\) be defined by
\[
C\xi = \xi - \mathbb{E}[\xi]
\]
for any integrable random variable \(\xi\). 
Applying \(C\) to the Bellman error gives
\[
C\delta_t = \delta_t - \mathbb{E}[\delta_t].
\]
Accordingly, the centered fixed-point condition becomes
\[
\mathbb{E}\big[(\delta_t-\mathbb{E}[\delta_t])\bm{\phi}(S_t)\big] = 0,
\]
which replaces the usual second-moment relation by a covariance-type condition. 
In particular, the corresponding centered TD solution satisfies
\[
-\bm{A}_{\mathrm{c}}\bm{\theta} + \bm{b}_{\mathrm{c}} = 0,
\qquad
\bm{A}_{\mathrm{c}} = \mathrm{Cov}\!\big(\bm{\phi}(S_t),\bm{\phi}(S_t)-\gamma\bm{\phi}(S_{t+1})\big),
\qquad
\bm{b}_{\mathrm{c}} = \mathrm{Cov}\!\big(R_{t+1},\bm{\phi}(S_t)\big).
\]

An online implementation introduces an auxiliary scalar \(\omega_t\) to track the mean corrected TD term: \(\omega_{t+1}=\omega_t+\alpha_t(\rho_t\delta_t-\omega_t)\), \(\bm{\theta}_{t+1}=\bm{\theta}_t+\alpha_t(\rho_t\delta_t-\omega_t)\bm{\phi}(S_t)\). Thus CTD can be viewed as an online realization of Bellman error centering.

Motivated by this perspective, we incorporate the centering idea into the emphatic TD framework. 
The resulting method, Centered Emphatic Temporal-Difference Learning (CETD), keeps the emphatic weighting mechanism of ETD while replacing the raw emphatic TD term by its centered version. 
Specifically, with \(F_t\) updated by \eqref{etdf}, CETD uses
\begin{equation}
\begin{aligned}
\bm{\theta}_{t+1}&=\bm{\theta}_t+\alpha_t\bigl(F_t\rho_t\delta_t-\omega_t\bigr)\bm{\phi}(S_t),\\
\omega_{t+1}&=\omega_t+\alpha_t\bigl(F_t\rho_t\delta_t-\omega_t\bigr).
\end{aligned}
\end{equation}

In this sense, CETD can be viewed as the emphatic counterpart of CTD, with the same centering principle applied after emphatic reweighting.

\subsection{The key matrix of CETD}

We next characterize the mean dynamics of CETD through its coupled single-timescale system. For \(\bm{z}_t=[\bm{\theta}_t^\top,\omega_t]^\top\) and \(\delta_t=R_{t+1}+\gamma\bm{\phi}(S_{t+1})^\top\bm{\theta}_t-\bm{\phi}(S_t)^\top\bm{\theta}_t\), the CETD update is
\begin{equation}
\bm{z}_{t+1}=\bm{z}_t+\alpha_t\bigl(\overline{\bm{h}}_t-\overline{\bm{G}}_t\bm{z}_t\bigr),
\label{eq:cetd-joint-update}
\end{equation}
where the drift components are
\[
\overline{\bm{h}}_t=\begin{bmatrix}F_t\rho_tR_{t+1}\bm{\phi}(S_t)\\ F_t\rho_tR_{t+1}\end{bmatrix},\qquad
\overline{\bm{G}}_t=\begin{bmatrix}
F_t\rho_t\bm{\phi}(S_t)(\bm{\phi}(S_t)-\gamma\bm{\phi}(S_{t+1}))^\top & \bm{\phi}(S_t)\\[1mm]
F_t\rho_t(\bm{\phi}(S_t)-\gamma\bm{\phi}(S_{t+1}))^\top & 1
\end{bmatrix}.
\]

Taking expectations and passing to the limit yield the key matrix
\begin{equation}
\overline{\bm{G}}_{\mathrm{CETD}}
=
\begin{bmatrix}
\bm{\Phi}^\top \bm{F}(\bm{I}-\gamma \bm{P}_\pi)\bm{\Phi} & \bm{\Phi}^\top \bm{d}_\mu\\[1mm]
\bm{d}_\mu^\top \bm{\Phi} & 1
\end{bmatrix}.
\label{eq:cetd-key}
\end{equation}

\subsection{A New Two-State Counterexample for CETD}

Consider the following two-state counterexample for CETD. 
Figure~\ref{fig:two-state-cetd} shows the behavior-policy transitions in solid lines and the target-policy transitions in dashed lines.

\begin{figure}[t]
\centering
\begin{tikzpicture}[>=latex, line width=0.95pt, scale=1]
\tikzstyle{state}=[circle, draw=black, minimum size=11mm, inner sep=0pt]

\node[state] (s1) at (0,0) {\(s_1\)};
\node[state] (s2) at (5.8,0) {\(s_2\)};

\draw[->, dashed] (s1) to[bend left=38]
node[midway, above=3pt, draw=black, dashed, circle, inner sep=1pt] { \(1.0\)} (s2);
\draw[->, dashed] (s2) to[loop above]
node[midway, above=3pt, draw=black, dashed, circle, inner sep=1pt] { \(1.0\)} (s2);

\draw[->] (s1) to[loop left] node[left] {\(0.05\)} (s1);
\draw[->] (s1) to[bend right=24] node[midway, below=2pt] {\(0.95\)} (s2);
\draw[->] (s2) to[bend right=24] node[midway, below=2pt] {\(0.05\)} (s1);
\draw[->] (s2) to[loop right] node[right] {\(0.95\)} (s2);

\node at (0,-1.3) {\(\bm{\phi}(s_1)=1\)};
\node at (5.8,-1.3) {\(\bm{\phi}(s_2)=0.6\)};
\end{tikzpicture}
\caption{A new two-state counterexample for CETD.}
\label{fig:two-state-cetd}
\end{figure}
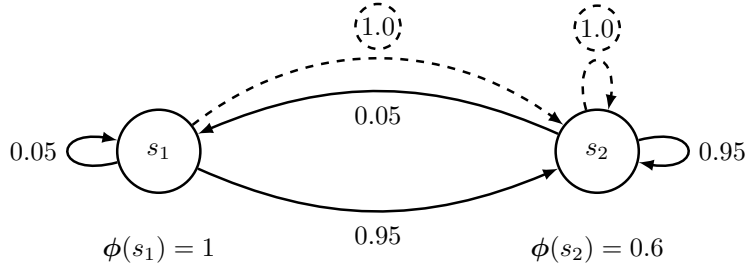

Let
\[
\bm{P}_\mu=
\begin{bmatrix}
0.05 & 0.95\\
0.05 & 0.95
\end{bmatrix},
\qquad
\bm{P}_\pi=
\begin{bmatrix}
0 & 1\\
0 & 1
\end{bmatrix},
\qquad
\bm{d}_\mu=
\begin{bmatrix}
0.05\\
0.95
\end{bmatrix},
\]
and
\[
\bm{\Phi}=
\begin{bmatrix}
1\\
0.6
\end{bmatrix},
\qquad
\gamma=0.9,
\qquad
\bm{r}_\pi=\bm{0}.
\]
Then \(\bm{d}_\mu\) is the stationary distribution of \(\bm{P}_\mu\), and
\[
\bm{f}
=
(\bm{I}-\gamma \bm{P}_\pi^\top)^{-1}\bm{d}_\mu
=
\begin{bmatrix}
0.05\\
9.95
\end{bmatrix},
\qquad
\bm{F}=\mathrm{diag}(\bm{f}).
\]

A direct calculation gives
\[
\bm{\Phi}^\top \bm{F}(\bm{I}-\gamma \bm{P}_\pi)\bm{\Phi}
=
0.3812,
\qquad
\bm{d}_\mu^\top \bm{\Phi}
=
0.62.
\]
Thus the CETD key matrix in this example is
\begin{equation}
\overline{\bm{G}}_{\mathrm{CETD}}
=
\begin{bmatrix}
0.3812 & 0.62\\
0.62 & 1
\end{bmatrix}.
\label{eq:cetd-key-two-state}
\end{equation}

To show that \(\overline{\bm{G}}_{\mathrm{CETD}}\) is not positive definite, take \(\bm{x}=[1,-0.62]^\top\). Then
\[
\bm{x}^\top \overline{\bm{G}}_{\mathrm{CETD}} \bm{x}
=
0.3812 - 2(0.62)(0.62) + 0.62^2
=
0.3812 - 0.3844
=
-0.0032 < 0.
\]
Therefore, \(\overline{\bm{G}}_{\mathrm{CETD}}\) is not positive definite in this example. 
Consequently, CETD loses the required positivity property even in this simple two-state setting.
\section{Regularized Emphatic Temporal-Difference Learning}\label{sec:retd}

The two-state counterexample above shows that CETD may lose the required positivity property under the coupled single-timescale formulation. The problem is not the emphatic trace itself: the follow-on trace is still the mechanism that changes the projection geometry. The unstable part is the auxiliary centering channel, which subtracts \(\omega_t\) from the value update and is itself driven by the same emphatic TD error. If this auxiliary variable reacts too strongly, centering introduces a destabilizing low-rank coupling into the joint system.

RETD therefore keeps the emphatic trace unchanged and adds damping only to the auxiliary recursion. The parameter \(c\) controls this damping. Small \(c\) leaves the method close to CETD, moderate \(c\) weakens the harmful coupling while preserving centering, and very large \(c\) drives \(\omega_t\) toward zero, making the update increasingly ETD-like. With \(F_t\) updated by \eqref{etdf} and \(c>0\), Regularized Emphatic Temporal-Difference Learning (RETD) uses
\begin{equation}
\begin{aligned}
\bm{\theta}_{t+1}&=\bm{\theta}_t+\alpha_t\bigl(F_t\rho_t\delta_t-\omega_t\bigr)\bm{\phi}(S_t),\\
\omega_{t+1}&=\omega_t+\alpha_t\bigl(F_t\rho_t\delta_t-(1+c)\omega_t\bigr).
\end{aligned}
\end{equation}

When \(c=0\), RETD reduces to CETD; for \(c>0\), the factor \((1+c)\) adds controlled damping to the \(\omega_t\)-update while preserving the emphatic weighting mechanism. Thus the regularizer should be read dynamically rather than variationally: it is not a penalty on the value function or Bellman error, but a conditioning term for the coupled \((\bm{\theta},\omega)\) ODE that prevents the auxiliary centering channel from overturning the positive definiteness supplied by emphatic weighting.

\subsection{The key matrix of RETD}

We now derive the coupled system matrix associated with RETD. For \(\bm{z}_t=[\bm{\theta}_t^\top,\omega_t]^\top\) and \(\delta_t=R_{t+1}+\gamma\bm{\phi}(S_{t+1})^\top\bm{\theta}_t-\bm{\phi}(S_t)^\top\bm{\theta}_t\), the RETD update is
\begin{equation}
\bm{z}_{t+1}=\bm{z}_t+\alpha_t\bigl(\overline{\bm{h}}_t-\overline{\bm{G}}_t\bm{z}_t\bigr),
\label{eq:retd-joint-update}
\end{equation}
where the drift components are
\[
\overline{\bm{h}}_t=\begin{bmatrix}F_t\rho_tR_{t+1}\bm{\phi}(S_t)\\ F_t\rho_tR_{t+1}\end{bmatrix},\qquad
\overline{\bm{G}}_t=\begin{bmatrix}
F_t\rho_t\bm{\phi}(S_t)(\bm{\phi}(S_t)-\gamma\bm{\phi}(S_{t+1}))^\top & \bm{\phi}(S_t)\\[1mm]
F_t\rho_t(\bm{\phi}(S_t)-\gamma\bm{\phi}(S_{t+1}))^\top & 1+c
\end{bmatrix}.
\]

Taking expectations and passing to the limit yield the key matrix
\begin{equation}
\overline{\bm{G}}_{\mathrm{RETD}}
=
\begin{bmatrix}
\bm{\Phi}^\top \bm{F}(\bm{I}-\gamma \bm{P}_\pi)\bm{\Phi} & \bm{\Phi}^\top \bm{d}_\mu\\[1mm]
\bm{d}_\mu^\top \bm{\Phi} & 1+c
\end{bmatrix}.
\label{eq:retd-key}
\end{equation}

\subsection{A two-state analysis}

Under the same two-state setting as in the CETD section, the expectation terms satisfy
\[
\bm{\Phi}^\top \bm{F}(\bm{I}-\gamma \bm{P}_\pi)\bm{\Phi}
=
0.3812,
\qquad
\bm{d}_\mu^\top \bm{\Phi}
=
0.62.
\]
Therefore the RETD key matrix in this example is
\begin{equation}
\overline{\bm{G}}_{\mathrm{RETD}}
=
\begin{bmatrix}
0.3812 & 0.62\\
0.62 & 1+c
\end{bmatrix}.
\label{eq:retd-key-two-state}
\end{equation}

Since \(\overline{\bm{G}}_{\mathrm{RETD}}\) is symmetric, positive definiteness is equivalent to positivity of the leading principal minors. The first minor is \(0.3812>0\), and
\[
\det(\overline{\bm{G}}_{\mathrm{RETD}})=0.3812(1+c)-0.62^2=0.3812(1+c)-0.3844.
\]
Thus \(\overline{\bm{G}}_{\mathrm{RETD}}\succ0\) iff \(0.3812(1+c)-0.3844>0\), equivalently
\begin{equation}
c>\frac{0.3844}{0.3812}-1\approx 0.0084.
\label{eq:c-range-retd}
\end{equation}
In this two-state example, a sufficiently small positive regularization already restores positivity in the coupled single-timescale analysis, showing that the failure is local to the auxiliary channel rather than to emphatic weighting itself.

\subsection{Convergence discussion}

The matrix expression in \eqref{eq:retd-key} shows that RETD forms a continuous bridge between CETD and ETD. 
When \(c=0\), the method reduces to CETD; as \(c\) increases, the lower-right entry of the coupled matrix grows and the system becomes more stable.

Therefore, once \(c\) is chosen so that \(\overline{\bm{G}}_{\mathrm{RETD}}\) is positive definite, the associated expected system matrix satisfies the required stability property. 
Under the standard stochastic approximation conditions together with the ergodicity assumptions stated in Section~\ref{sec:theory}, the RETD iterates converge to the unique equilibrium of the associated ODE, as formalized next.
\section{Theoretical Analysis}
\label{sec:theory}

In this section, we establish the convergence properties of the proposed RETD algorithm. Our analysis builds upon the Ordinary Differential Equation (ODE) framework for stochastic approximation algorithms. To rigorously analyze the asymptotic behavior, we require the following standard assumptions regarding the environment, the step sizes, and the structural regularization.

\begin{assumption}[Markov Chain Regularity]
\label{assum:markov}
The state space \(\mathcal{S}\) and action space \(\mathcal{A}\) are finite. The stochastic process \(\{Y_t\}\) induced by the behavior policy \(\mu\) forms a geometrically ergodic Markov chain with a unique stationary distribution \(d_Y\).
\end{assumption}

\begin{assumption}[Learning Rates]
\label{assum:lr}
The step-size sequence \(\{\alpha_t\}\) is positive, non-increasing, and satisfies the standard Robbins-Monro conditions
\[
\sum_{t=0}^{\infty} \alpha_t = \infty,
\qquad
\sum_{t=0}^{\infty} \alpha_t^2 < \infty.
\]
\end{assumption}

\begin{assumption}[Structural Regularization]
\label{assum:regularization}
The regularization parameter \(c\) satisfies
\[
c \ge \frac{\gamma}{1-\gamma}.
\]
\end{assumption}

It is worth emphasizing that this lower bound is a conservative sufficient condition for the general convergence analysis, not a tuning rule or a sharp threshold in every problem instance. For example, in the two-state calculation in Equation~\ref{eq:c-range-retd}, the RETD key matrix is already positive definite for \(c>0.0084\). The fixed-stepsize experiments therefore do not claim to verify the theorem by choosing \(c\ge\gamma/(1-\gamma)\); instead, they test whether the same regularized-coupling mechanism remains useful under practical, problem-dependent choices of \(c\).

The role of \(c\) should therefore be read at two levels. The theorem uses a problem-independent bound to make the proof uniform over finite off-policy prediction problems. In a concrete instance, however, the relevant threshold is governed by the actual coupled matrix in \eqref{eq:retd-key}. The two-state analysis gives an example of this sharper problem-dependent calculation, while the empirical \(c\)-scans in Appendix~\ref{app:experiment_details} provide a practical way to locate the intermediate regime where centering is active but the auxiliary recursion is sufficiently damped.

Under the single-timescale setting, both the main parameter \(\bm{\theta}_t\) and the auxiliary variable \(\omega_t\) are updated synchronously using the same learning rate sequence \(\{\alpha_t\}\). Equivalently, the RETD iteration evolves in the joint variable
\[
\bm{z}_t \doteq
\begin{bmatrix}
\bm{\theta}_t\\
\omega_t
\end{bmatrix}.
\]
Based on the ergodicity of the Markov chain, the Lipschitz continuity of the update function, and the structural condition imposed by \(c\), we present the main convergence theorem for RETD.

\begin{theorem}[Convergence of RETD]
\label{thm:main_convergence}
Let Assumptions \ref{assum:markov}, \ref{assum:lr}, and \ref{assum:regularization} hold. The joint iterative sequence \(\{\bm{z}_t\}\) generated by the RETD update rule is almost surely bounded, i.e.,
\[
\sup_t \|\bm{z}_t\| < \infty
\qquad
\text{a.s.}
\]
Furthermore, \(\{\bm{z}_t\}\) converges almost surely to the unique globally asymptotically stable equilibrium of the associated ODE
\[
\dot{\bm{z}}(t)=\overline{\bm{h}}_{\mathrm{RETD}}-\overline{\bm{G}}_{\mathrm{RETD}}\bm{z}(t),
\]
where \(\overline{\bm{G}}_{\mathrm{RETD}}\) is given by \eqref{eq:retd-key}.
\end{theorem}

\begin{proof}[Proof Sketch]
The proof is based on the ODE method for stochastic approximation with Markovian noise \citep{liu2025ode}. Since the RETD update is affine in the joint variable \(\bm{z}_t\), the associated mean field takes the form
\[
\overline{\bm{h}}_{\mathrm{RETD}}-\overline{\bm{G}}_{\mathrm{RETD}}\bm{z}.
\]
The key step is to verify that \(\overline{\bm{G}}_{\mathrm{RETD}}\) is positive definite under Assumption \ref{assum:regularization}. Intuitively, completing the square isolates the auxiliary variable \(\omega\), while the remaining feature-space term becomes positive because the regularizer makes the residual matrix an \(M\)-matrix with positive row and column sums. Positive definiteness implies that the drift matrix \(-\overline{\bm{G}}_{\mathrm{RETD}}\) is Hurwitz, so the limiting ODE has a unique globally asymptotically stable equilibrium.

The remaining regularity conditions follow from the finiteness of the state-action space, the geometric ergodicity of the Markov chain, and the standard properties of the step-size sequence. The complete proof is deferred to Appendix \ref{app:convergence_proof}.
\end{proof}
\section{Experimental studies}
\label{sec:experiments}

The role of these experiments is diagnostic rather than competitive: they isolate the structural mechanism analyzed in Sections~\ref{sec:cetd}--\ref{sec:theory}, namely that naive Bellman-error centering destabilizes the emphatic coupled system, whereas regularizing only the auxiliary centering channel restores usable behavior without replacing the emphatic projection geometry. The evidence is organized around three questions: stability under aggressive off-policy bootstrapping, fixed-point geometry once the iterates remain bounded, and the structural role of the regularization parameter \(c\). The suite is therefore used to isolate mechanisms rather than to claim broad benchmark dominance.

\paragraph{Protocol.}
We use seven linear off-policy prediction tasks: random walk with tabular, inverted, and dependent features, Boyan chain~\citep{boyan1999least}, the two-state counterexample associated with our centered-emphatic analysis, Baird's seven-state counterexample~\citep{baird1995residual}, and the new two-state task introduced to compare CETD and RETD. The baselines are off-policy TD, GTD2, TDC, and TDRC~\citep{sutton2008convergent,sutton2009fast,ghiassian2020gradient}, together with ETD~\citep{sutton2016emphatic}, truncated ETD (TETD)~\citep{zhang2022truncated}, and CETD where defined. The error metric is RMSE against the true value function, or against the zero target in zero-reward counterexamples. Curves are cross-run means with one-standard-deviation bands and no temporal smoothing. The main text reports results at the common stepsize \(\alpha=0.01\); full run counts, horizons, additional stepsizes, and parameter scans are deferred to Appendix~\ref{app:experiment_details}.

\subsection{Main diagnostics: stability and geometry}

Figure~\ref{fig:main_algorithm_comparison} summarizes the two diagnostics carried in the main text. Boyan chain is a standard linear prediction benchmark on which emphatic weighting yields a favorable approximation geometry: RETD tracks the best ETD-style behavior rather than reverting to the ordinary TD geometry, indicating that regularization does not erase the emphatic fixed-point motivation. Baird's counterexample is the strongest off-policy stress test in the suite. At \(\alpha=0.01\), RETD remains numerically stable while ETD and TETD diverge by the criterion of Table~\ref{tab:divergence_rate}; their trajectories are therefore not plotted in panel~(b) and are reported as \textsc{Div.} in Table~\ref{tab:tail_rmse}.

\begin{figure}[t]
\centering
\includegraphics[width=0.62\linewidth]{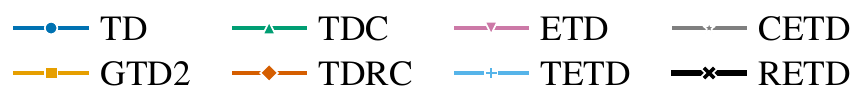}
\vspace{0.25em}

\begin{subfigure}[t]{0.46\linewidth}
  \centering
  \includegraphics[width=\linewidth]{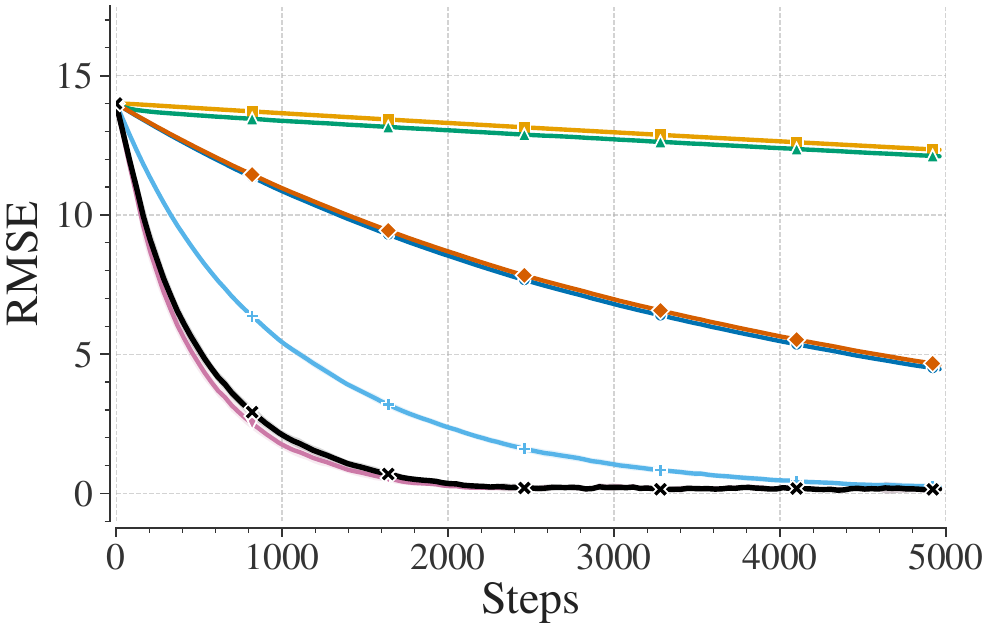}
  \caption{Boyan chain}
\end{subfigure}\hfill
\begin{subfigure}[t]{0.46\linewidth}
  \centering
  \includegraphics[width=\linewidth]{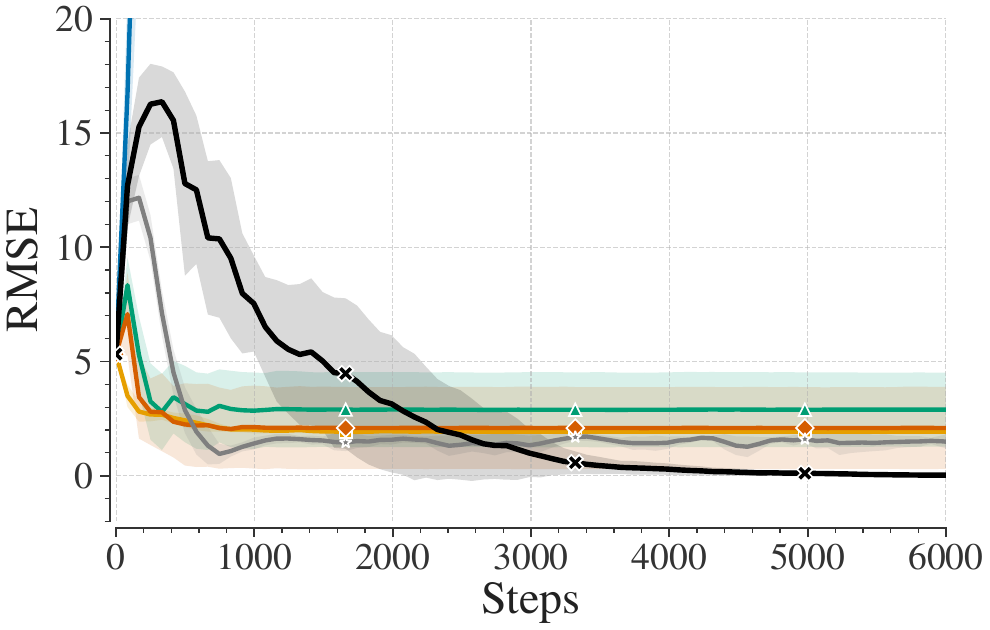}
  \caption{Baird's counterexample}
\end{subfigure}
\caption{Main diagnostic comparisons at \(\alpha=0.01\). Panel~(a) shows the geometry diagnostic on Boyan chain; panel~(b) shows the off-policy stability diagnostic on Baird, where ETD and TETD trajectories are numerically unusable at this stepsize and are reported in Table~\ref{tab:divergence_rate} rather than plotted.}
\label{fig:main_algorithm_comparison}
\end{figure}


\begin{table}[t]
\centering
\footnotesize
\setlength{\tabcolsep}{3.0pt}
\caption{Tail-average RMSE at $\alpha=0.01$. For readability, algorithms are split into two blocks. For each independent run, we average RMSE over the final 10\% of update steps; cells report mean $\pm$ standard deviation across runs. Diverged methods are marked as \textsc{Div.} and excluded from bolding. Bold entries denote the best non-diverged mean and all methods within one standard deviation of it.}
\label{tab:tail_rmse}
\begin{tabular}{@{}lcccc@{}}
\toprule
Environment & TD & GTD2 & TDC & TDRC \\
\midrule
RW tabular & $0.064\pm 0.013$ & $0.272\pm 6.96\!\times\!10^{-3}$ & $0.242\pm 7.54\!\times\!10^{-3}$ & $0.069\pm 0.013$ \\
RW inverted & $0.097\pm 0.010$ & $0.176\pm 5.90\!\times\!10^{-3}$ & $0.113\pm 6.54\!\times\!10^{-3}$ & $0.103\pm 5.37\!\times\!10^{-3}$ \\
RW dependent & $0.170\pm 7.00\!\times\!10^{-3}$ & $0.209\pm 7.49\!\times\!10^{-3}$ & $0.215\pm 0.013$ & $0.175\pm 0.012$ \\
Boyan chain & $4.666\pm 0.054$ & $12.40\pm 2.83\!\times\!10^{-3}$ & $12.17\pm 0.040$ & $4.832\pm 0.053$ \\
Two-state & \textsc{Div.} & $1.457\pm 0.055$ & $1.380\pm 0.348$ & $0.820\pm 0.100$ \\
Baird & \textsc{Div.} & $1.927\pm 6.55\!\times\!10^{-3}$ & $2.883\pm 1.617$ & $2.092\pm 1.801$ \\
New two-state & $8.573\pm 0.067$ & $8.543\pm 0.046$ & $8.574\pm 0.037$ & $8.582\pm 0.034$ \\
\bottomrule
\end{tabular}
\vspace{0.35em}
\begin{tabular}{@{}lcccc@{}}
\toprule
Environment & ETD & TETD & CETD & RETD \\
\midrule
RW tabular & $0.040\pm 0.015$ & $\mathbf{0.023\pm 4.02\!\times\!10^{-3}}$ & $0.030\pm 6.63\!\times\!10^{-3}$ & $0.037\pm 5.57\!\times\!10^{-3}$ \\
RW inverted & $0.047\pm 0.011$ & $\mathbf{0.035\pm 7.72\!\times\!10^{-3}}$ & $0.047\pm 0.011$ & $0.046\pm 0.011$ \\
RW dependent & $\mathbf{0.129\pm 0.016}$ & $\mathbf{0.126\pm 0.016}$ & $\mathbf{0.131\pm 0.015}$ & $\mathbf{0.121\pm 0.015}$ \\
Boyan chain & $\mathbf{0.172\pm 0.108}$ & $0.282\pm 0.091$ & -- & $\mathbf{0.166\pm 0.105}$ \\
Two-state & $\mathbf{(1.14\!\pm\!2.95)\!\times\!10^{-30}}$ & $\mathbf{(1.07\!\pm\!2.90)\!\times\!10^{-19}}$ & $(7.01\!\pm\!20.91)\!\times\!10^{-5}$ & $\mathbf{(1.63\!\pm\!4.89)\!\times\!10^{-21}}$ \\
Baird & \textsc{Div.} & \textsc{Div.} & $1.529\pm 0.153$ & $\mathbf{(1.41\!\pm\!2.11)\!\times\!10^{-4}}$ \\
New two-state & \textsc{Div.} & $5.015\pm 0.457$ & $43.73\pm 6.459$ & $\mathbf{4.131\pm 0.732}$ \\
\bottomrule
\end{tabular}
\end{table}

\begin{table}[t]
\centering
\footnotesize
\setlength{\tabcolsep}{2.5pt}
\caption{Stability details at $\alpha=0.01$ for environments with at least one divergent method (other environments have zero diverged runs for all algorithms and are omitted). A run is counted as diverged if its RMSE trajectory contains NaN/Inf or if its maximum RMSE exceeds $10^3$. Each cell reports the maximum RMSE observed across the 10 runs of that algorithm-environment pair; entries with a \textsuperscript{\textsc{Div.}\,k/n} superscript indicate \(k\) of the \(n\) runs were classified as diverged. Bold entries denote, within each row, the smallest max-RMSE among Stable methods and any Stable method within \(5\%\) of it; diverged methods are excluded from bolding.}
\label{tab:divergence_rate}
\begin{tabular}{@{}lcccccccc@{}}
\toprule
Environment & TD & GTD2 & TDC & TDRC & ETD & TETD & CETD & RETD \\
\midrule
Two-state & $7.10\!\times\!10^{4}$\textsuperscript{\textsc{Div.}\,10/10} & $\mathbf{1.697}$ & $2.659$ & $1.916$ & $\mathbf{1.720}$ & $1.809$ & $2.043$ & $\mathbf{1.754}$ \\
Baird & $2.21\!\times\!10^{12}$\textsuperscript{\textsc{Div.}\,10/10} & $\mathbf{5.318}$ & $11.50$ & $11.35$ & $2.41\!\times\!10^{12}$\textsuperscript{\textsc{Div.}\,50/50} & $5.44\!\times\!10^{10}$\textsuperscript{\textsc{Div.}\,50/50} & $17.54$ & $21.16$ \\
New two-state & $\mathbf{8.802}$ & $\mathbf{8.735}$ & $\mathbf{8.794}$ & $\mathbf{8.790}$ & $3.89\!\times\!10^{3}$\textsuperscript{\textsc{Div.}\,1/10} & $\mathbf{8.747}$ & $64.92$ & $\mathbf{8.758}$ \\
\bottomrule
\end{tabular}
\end{table}

Tables~\ref{tab:tail_rmse} and~\ref{tab:divergence_rate} provide the quantitative counterpart to Figure~\ref{fig:main_algorithm_comparison}. The tail-average RMSE table averages the final 10\% of training steps and so is robust to a single noisy iterate. Diverged methods are reported as \textsc{Div.} and excluded from the bolding rule, which makes the Baird failure of ETD and TETD explicit instead of hiding it behind enormous finite RMSE values. Across the non-diverged methods, RETD is the only emphatic-style method that is simultaneously stable and accurate on Baird, while remaining competitive with the best baselines on the standard environments.

\subsection{Why RETD is needed beyond CETD}

The closest alternative to RETD is CETD: both apply centering to the auxiliary recursion, but only RETD adds the structural regularization analyzed in Section~\ref{sec:theory}. The new two-state row of Table~\ref{tab:tail_rmse} isolates the difference. CETD attains a tail RMSE of \(43.73\) and a maximum RMSE of \(64.92\) (Table~\ref{tab:divergence_rate}), markedly worse than the unbiased baselines whose tail RMSE sits near \(8.5\) on this 1D scalar-feature task; RETD reaches \(4.131\), the smallest tail RMSE of all eight methods. Both CETD and RETD remain numerically bounded, so this is a clean comparison of fixed-point geometry rather than of stability. RETD also matches or slightly improves on ETD (\(5.015\)) and TETD (\(8.543\)) on this row, confirming that the regularizer repairs CETD without surrendering the emphatic geometry. This evidence is the central justification for keeping the regularization term: it turns Bellman-error centering into a usable component of an emphatic update and distinguishes RETD from a naive centered emphatic construction without regularization.

\subsection{Robustness and parameter interpretation}

Appendix~\ref{app:experiment_details} contains the complete seven-environment curves at \(\alpha=0.01\), additional learning-rate slices, and RETD scans over both \(\alpha\) and \(c\). The qualitative pattern across these scans is consistent and structural: very small \(c\) approaches CETD and inherits the indefinite coupled matrix, very large \(c\) damps the auxiliary recursion so strongly that RETD behaves increasingly like ETD, and an intermediate range delivers stable centered emphatic learning. The parameter \(c\) should therefore be read as a structural control knob between CETD-like and ETD-like regimes, not as a conventional penalty whose largest value is best; the practical fixed-stepsize choices used here (\(c=9\) on the standard environments, \(c=0.5\) on Baird) sit in this intermediate range and are not attempts to satisfy the conservative sufficient bound \(c\geq\gamma/(1-\gamma)\).

\subsection{Scope of comparison and limitations}

These diagnostics deliberately stay within linear off-policy prediction so that the matrix mechanism behind RETD can be tested in isolation. For the same reason, we do not directly benchmark against deep or learned-emphasis variants~\citep{jiang2021emphatic,jiang2022learning}, which combine nonlinear function approximation, replay, and an auxiliary emphasis learner; such systems vary along several dimensions at once and would obscure rather than test the linear key-matrix property that motivates RETD. Extending the regularized centering principle to nonlinear function approximation, eligibility traces, and large-scale control is left to future work.
\section{Conclusions and future work}
\label{sec:conclusion}

We presented RETD, a regularized centered extension of emphatic temporal-difference learning. Bellman-error centering is a natural way to reduce drift in TD errors, but applying it naively to ETD destroys the matrix property that makes the emphatic update analytically well behaved. RETD repairs this failure by regularizing only the auxiliary centering recursion, preserving the emphatic trace mechanism while controlling the low-rank coupling introduced by centering. We derived the associated coupled matrix, showed how the regularizer restores positive definiteness, and established convergence through a stochastic-approximation ODE argument. The diagnostic experiments support this view: RETD avoids the instability of naive centered emphatic learning while remaining competitive with standard TD, gradient-TD, and emphatic baselines, and the parameter \(c\) exhibits a meaningful intermediate regime rather than a monotone ``larger-is-better'' behavior.

The broader lesson is that emphatic weighting and centering interact through the joint matrix, not only through the variance of the follow-on trace. The two-state calculations make this interaction explicit: the emphatic block may remain positive while the auxiliary off-diagonal coupling is large enough to make the centered system indefinite. RETD addresses precisely this coupling, so its regularization parameter should be interpreted through the conditioning of the coupled ODE rather than through the usual objective-regularization lens. This also explains why conservative theorem-level bounds and useful fixed-stepsize choices need not coincide; the former guarantee a uniform proof, while the latter are governed by the problem-specific key matrix and are best read together with the empirical \(c\)-scans reported above.

Three directions are particularly natural for future work. First, scaling the empirical study to substantially larger off-policy prediction problems would test whether the regularized centering mechanism continues to improve stability outside the diagnostic regime considered here. Second, extending RETD to \(\lambda\)-returns, yielding an RETD\((\lambda)\) family, would connect the regularized centering principle to the broader eligibility-trace machinery used in modern emphatic methods. Third, the regularization parameter \(c\) is currently fixed; an adaptive schedule that increases \(c\) gradually from a small initial value, rather than committing to a single constant, may better track the evolving conditioning of the coupled system and make tuning more principled.
\FloatBarrier
\clearpage



\appendix
\section{Convergence Proof of RETD}
\label{app:convergence_proof}

In this appendix, we provide the detailed convergence proof for RETD. Our proof is based on the Ordinary Differential Equation (ODE) method for stochastic approximation with Markovian noise developed by \citep{liu2025ode}. We first restate the general framework and the convergence result that we will invoke, and then verify that RETD satisfies all the required conditions.

\subsection[ODE framework]{ODE framework from \citet{liu2025ode}}

Consider a stochastic approximation recursion of the form
\begin{equation}
\bm{z}_{n+1}
=
\bm{z}_n + \alpha(n) H(\bm{z}_n, Y_{n+1}),
\qquad n=0,1,\dots,
\label{eq:liu-sa}
\end{equation}
where \(\bm{z}_n \in \mathbb{R}^m\) is the iterate, \(\{\alpha(n)\}\) is a deterministic step-size sequence, and \(\{Y_n\}\) is a Markov chain representing the noise process.

For \(c \ge 1\), define the rescaled update
\[
H_c(\bm{z},y) \doteq \frac{H(c\bm{z},y)}{c}.
\]
Following \citep{liu2025ode}, the convergence analysis relies on the following assumptions.

\begin{assumption}[Markov chain ergodicity]
\label{assum:liu_markov}
The Markov chain \(\{Y_n\}\) has a unique invariant probability measure \(d_Y\).
\end{assumption}

\begin{assumption}[Learning rates]
\label{assum:liu_lr}
The step-size sequence \(\{\alpha(n)\}\) is positive, decreasing, and satisfies
\[
\sum_{n=0}^{\infty}\alpha(n)=\infty,
\qquad
\lim_{n\to\infty}\alpha(n)=0,
\qquad
\left|\frac{\alpha(n)-\alpha(n+1)}{\alpha(n)}\right|=\mathcal{O}(\alpha(n)).
\]
\end{assumption}

\begin{assumption}[Scaled limit structure]
\label{assum:liu_limit_func}
There exist a measurable function \(H_\infty(\bm{z},y)\), a scalar function \(\kappa(c)\) with \(\lim_{c\to\infty}\kappa(c)=0\), and a measurable function \(b(\bm{z},y)\) such that
\[
H_c(\bm{z},y)-H_\infty(\bm{z},y)=\kappa(c)b(\bm{z},y),
\]
and \(b(\bm{z},y)\) is Lipschitz continuous in \(\bm{z}\) with an integrable Lipschitz coefficient.
\end{assumption}

\begin{assumption}[Lipschitz continuity and finite expectations]
\label{assum:liu_lipschitz}
Both \(H(\bm{z},y)\) and \(H_\infty(\bm{z},y)\) are Lipschitz continuous in \(\bm{z}\), with a measurable Lipschitz coefficient \(L(y)\). Moreover,
\[
h(\bm{z}) \doteq \mathbb{E}_{y\sim d_Y}[H(\bm{z},y)],
\qquad
h_\infty(\bm{z}) \doteq \mathbb{E}_{y\sim d_Y}[H_\infty(\bm{z},y)]
\]
are well defined for all \(\bm{z}\), and \(\mathbb{E}_{y\sim d_Y}[L(y)]<\infty\).
\end{assumption}

\begin{assumption}[Global asymptotic stability]
\label{assum:liu_stability}
The rescaled mean field \(h_c(\bm{z}) \doteq h(c\bm{z})/c\) converges to \(h_\infty(\bm{z})\) uniformly on compact subsets as \(c\to\infty\). In addition, the limiting ODE
\[
\dot{\bm{z}}(t)=h_\infty(\bm{z}(t))
\]
has the origin \(\bm{0}\) as a globally asymptotically stable equilibrium.
\end{assumption}

\begin{assumption}[Asymptotic noise vanishing]
\label{assum:liu_noise}
For each relevant function \(g\) appearing in the framework, the cumulative fluctuation of \(g(Y_n)\) around its stationary mean vanishes at the step-size scale almost surely.
\end{assumption}

The following theorem is adapted from Theorem~7 of Liu et al.~(2025).

\begin{theorem}[\citep{liu2025ode}, adapted]
\label{thm:liu_general}
Suppose Assumptions \ref{assum:liu_markov}--\ref{assum:liu_stability} hold, together with the asymptotic noise condition in Assumption \ref{assum:liu_noise}. Then the stochastic approximation iterates generated by \eqref{eq:liu-sa} are stable, i.e.,
\[
\sup_n \|\bm{z}_n\| < \infty
\qquad \text{a.s.}
\]
\end{theorem}

Once stability is established, the convergence conclusion follows from the corresponding ODE method. The following corollary is adapted from Corollary~8 of \citep{liu2025ode}.

\begin{corollary}[Liu et al.~(2025), adapted]
\label{cor:liu_general}
Under the assumptions of Theorem \ref{thm:liu_general}, the iterates \(\{\bm{z}_n\}\) converge almost surely to a sample-path-dependent bounded invariant set of the ODE
\[
\dot{\bm{z}}(t)=h(\bm{z}(t)).
\]
In particular, if this ODE has a unique globally asymptotically stable equilibrium, then \(\{\bm{z}_n\}\) converges almost surely to that equilibrium.
\end{corollary}

\subsection{Application to RETD}

For RETD, define the joint iterate
\[
\bm{z}_t
\doteq
\begin{bmatrix}
\bm{\theta}_t\\
\omega_t
\end{bmatrix}.
\]
Then the RETD recursion can be written in the stochastic approximation form
\begin{equation}
\bm{z}_{t+1}
=
\bm{z}_t + \alpha_t\bigl(\overline{\bm{h}}_t-\overline{\bm{G}}_t\bm{z}_t\bigr),
\label{eq:retd-sa-form}
\end{equation}
where
\[
\overline{\bm{h}}_t
=
\begin{bmatrix}
F_t\rho_t R_{t+1}\bm{\phi}(S_t)\\
F_t\rho_t R_{t+1}
\end{bmatrix},
\qquad
\overline{\bm{G}}_t
=
\begin{bmatrix}
F_t\rho_t \bm{\phi}(S_t)\bigl(\bm{\phi}(S_t)-\gamma\bm{\phi}(S_{t+1})\bigr)^\top & \bm{\phi}(S_t)\\[1mm]
F_t\rho_t \bigl(\bm{\phi}(S_t)-\gamma\bm{\phi}(S_{t+1})\bigr)^\top & 1+c
\end{bmatrix}.
\]

Thus, with \(Y_{t+1}\) denoting the underlying Markovian sample at time \(t+1\), we may write
\[
H(\bm{z}_t,Y_{t+1})=\overline{\bm{h}}_t-\overline{\bm{G}}_t\bm{z}_t.
\]

We now verify the assumptions of Theorem \ref{thm:liu_general}.

\subsubsection{Regularity assumptions}

For a finite state-action MDP, the Markov chain induced by the behavior policy is geometrically ergodic under Assumption \ref{assum:markov} in the main text, and therefore Assumption \ref{assum:liu_markov} holds.

Assumption \ref{assum:liu_lr} follows directly from Assumption \ref{assum:lr} together with standard choices of diminishing step sizes. Since the RETD update is affine in \(\bm{z}\), both \(H(\bm{z},y)\) and its scaled limit are globally Lipschitz in \(\bm{z}\), so Assumptions \ref{assum:liu_limit_func} and \ref{assum:liu_lipschitz} are satisfied.

Finally, because the state-action space is finite and the required stationary averages involving the follow-on trace are well defined under the standing assumptions, the asymptotic noise condition in Assumption \ref{assum:liu_noise} is satisfied by the standard law-of-large-numbers arguments used in the Markovian-noise framework.

\subsubsection{Mean field and limiting ODE}

Taking expectations in \eqref{eq:retd-sa-form} and passing to the limit give the mean field
\[
h(\bm{z})
=
\overline{\bm{h}}_{\mathrm{RETD}}
-
\overline{\bm{G}}_{\mathrm{RETD}}\bm{z},
\]
where
\begin{equation}
\overline{\bm{G}}_{\mathrm{RETD}}
=
\begin{bmatrix}
\bm{\Phi}^\top \bm{F}(\bm{I}-\gamma \bm{P}_\pi)\bm{\Phi} & \bm{\Phi}^\top \bm{d}_\mu\\[1mm]
\bm{d}_\mu^\top \bm{\Phi} & 1+c
\end{bmatrix}.
\label{eq:appendix-retd-key}
\end{equation}
Hence the associated ODE is
\begin{equation}
\dot{\bm{z}}(t)
=
\overline{\bm{h}}_{\mathrm{RETD}}
-
\overline{\bm{G}}_{\mathrm{RETD}}\bm{z}(t).
\label{eq:appendix-ode}
\end{equation}

To verify Assumption \ref{assum:liu_stability}, it remains to show that the drift matrix \(-\overline{\bm{G}}_{\mathrm{RETD}}\) is Hurwitz. Since \(\overline{\bm{G}}_{\mathrm{RETD}}\) is symmetric, it suffices to prove that it is positive definite.

\subsection{Positive definiteness of the RETD matrix}

We first isolate the coupling identity used in the lower-left block of \(\overline{\bm{G}}_{\mathrm{RETD}}\).

\begin{proposition}
\label{prop:coupling}
The emphatic weighting vector satisfies
\[
\bm{f}^\top(\bm{I}-\gamma \bm{P}_\pi)=\bm{d}_\mu^\top.
\]
Equivalently,
\[
\mathbb{E}\!\left[F_t\rho_t\bigl(\bm{\phi}(S_t)-\gamma\bm{\phi}(S_{t+1})\bigr)\right]
=
\bm{d}_\mu^\top\bm{\Phi}.
\]
\end{proposition}

\begin{proof}
By definition,
\[
\bm{f}=(\bm{I}-\gamma \bm{P}_\pi^\top)^{-1}\bm{d}_\mu.
\]
Multiplying both sides by \((\bm{I}-\gamma \bm{P}_\pi^\top)\) yields
\[
(\bm{I}-\gamma \bm{P}_\pi^\top)\bm{f}=\bm{d}_\mu.
\]
Transposing gives
\[
\bm{f}^\top(\bm{I}-\gamma \bm{P}_\pi)=\bm{d}_\mu^\top.
\]
Multiplying by \(\bm{\Phi}\) on the right gives the equivalent identity.
\end{proof}

For any \(\bm{z}=[\bm{\theta}^\top,\omega]^\top\), we compute
\begin{align}
\bm{z}^\top \overline{\bm{G}}_{\mathrm{RETD}} \bm{z}
&=
\bm{\theta}^\top \bm{\Phi}^\top \bm{F}(\bm{I}-\gamma \bm{P}_\pi)\bm{\Phi}\bm{\theta}
+
2\omega\,\bm{d}_\mu^\top\bm{\Phi}\bm{\theta}
+
(1+c)\omega^2
\nonumber\\
&=
(1+c)\left(\omega+\frac{\bm{d}_\mu^\top\bm{\Phi}\bm{\theta}}{1+c}\right)^2
+
\bm{\theta}^\top \bm{\Phi}^\top \bm{K}\bm{\Phi}\bm{\theta},
\label{eq:appendix-complete-square}
\end{align}
where
\begin{equation}
\bm{K}
\doteq
\bm{F}(\bm{I}-\gamma \bm{P}_\pi)
-
\frac{1}{1+c}\bm{d}_\mu\bm{d}_\mu^\top.
\label{eq:appendix-K}
\end{equation}

Thus, it suffices to prove that \(\bm{\Phi}^\top\bm{K}\bm{\Phi}\) is positive definite.

First, for any \(i\neq j\),
\[
K_{ij}
=
-\gamma f(i)[P_\pi]_{ij}
-
\frac{1}{1+c}d_\mu(i)d_\mu(j)
\le 0.
\]

Second, the column sum satisfies
\[
\bm{1}^\top\bm{K}
=
\bm{f}^\top(\bm{I}-\gamma\bm{P}_\pi)
-
\frac{1}{1+c}\bm{d}_\mu^\top
=
\frac{c}{1+c}\bm{d}_\mu^\top
>
\bm{0}^\top,
\]
where Proposition \ref{prop:coupling} was used.

Third, the row sum satisfies
\begin{align}
\bm{K}\bm{1}
&=
\bm{F}(\bm{I}-\gamma \bm{P}_\pi)\bm{1}
-
\frac{1}{1+c}\bm{d}_\mu(\bm{d}_\mu^\top\bm{1})
\nonumber\\
&=
(1-\gamma)\bm{f}
-
\frac{1}{1+c}\bm{d}_\mu.
\label{eq:appendix-row-sum}
\end{align}
Using
\[
\bm{f}
=
(\bm{I}-\gamma\bm{P}_\pi^\top)^{-1}\bm{d}_\mu
=
\bm{d}_\mu + \sum_{t=1}^{\infty}(\gamma\bm{P}_\pi^\top)^t\bm{d}_\mu,
\]
we obtain
\[
\bm{K}\bm{1}
=
\left(\frac{c}{1+c}-\gamma\right)\bm{d}_\mu
+
(1-\gamma)\sum_{t=1}^{\infty}(\gamma\bm{P}_\pi^\top)^t\bm{d}_\mu.
\]
Under Assumption \ref{assum:regularization} in the main text,
\[
c \ge \frac{\gamma}{1-\gamma},
\]
which implies
\[
\frac{c}{1+c}-\gamma \ge 0.
\]
Hence \(\bm{K}\bm{1}>\bm{0}\).

Therefore, \(\bm{K}\) has non-positive off-diagonal entries and strictly positive row and column sums, so it is a nonsingular \(M\)-matrix. Consequently, \(\bm{\Phi}^\top\bm{K}\bm{\Phi}\) is positive definite, and \eqref{eq:appendix-complete-square} implies
\[
\overline{\bm{G}}_{\mathrm{RETD}} \succ 0.
\]

It follows that \(-\overline{\bm{G}}_{\mathrm{RETD}}\) is Hurwitz, and therefore the ODE \eqref{eq:appendix-ode} has a unique globally asymptotically stable equilibrium. Thus Assumption \ref{assum:liu_stability} holds for RETD.

\subsection{Conclusion of the proof}

By Theorem \ref{thm:liu_general}, the RETD iterates are almost surely bounded. By Corollary \ref{cor:liu_general}, since the ODE \eqref{eq:appendix-ode} has a unique globally asymptotically stable equilibrium, the joint iterate sequence \(\{\bm{z}_t\}\) converges almost surely to that equilibrium. This completes the proof of Theorem \ref{thm:main_convergence}.
\section{Experimental details}
\label{app:experiment_details}

This appendix records the protocol behind the diagnostics in Section~\ref{sec:experiments} and provides the full coverage and reproducibility material that the main text deliberately omits. It is structured as: (i) environments and plotted methods, (ii) learning rates and regularization parameters, (iii) statistics shown in the plots, (iv) computing resources, (v) full coverage at the main-text stepsize, (vi) additional learning-rate slices, (vii) RETD robustness scans, and (viii) a structural reading of large \(c\). The code and regenerated figures are contained in the accompanying experiment-results directory; algorithm-comparison curves are read from \texttt{results\_b/}, with no-legend copies under \texttt{paper\_figures/} so the main text and appendix can share the legend in Figure~\ref{fig:main_algorithm_comparison}, and RETD scans are read from \texttt{results\_a/}.

\subsection{Environments and plotted methods}

\paragraph{Random walk.}
The random-walk prediction task uses three feature representations: tabular, inverted, and dependent. For each representation, the comparison includes TD, GTD2, TDC, TDRC, ETD, TETD, CETD, and RETD, with \(10\) independent runs and \(5000\) update steps.

\paragraph{Boyan chain.}
The Boyan chain experiment follows~\citet{boyan1999least}. The plotted methods are TD, GTD2, TDC, TDRC, ETD, TETD, and RETD; CETD is not included for this environment. Stored runs use \(10\) independent repetitions and \(5000\) update steps.

\paragraph{Two-state counterexample.}
The two-state counterexample evaluated in Tables~\ref{tab:tail_rmse}--\ref{tab:divergence_rate} and Figure~\ref{fig:appendix_alpha_0p01} is the classical Sutton off-policy two-state task: states \(\{s_1,s_2\}\) with scalar features \(\phi(s_1)=1,\phi(s_2)=2\), uniform behavior over the two actions, target policy that always picks the action moving the chain to \(s_2\), zero reward, and \(\gamma=0.9\); the comparison includes TD, GTD2, TDC, TDRC, ETD, TETD, CETD, and RETD with \(10\) repetitions and \(5000\) update steps. The CETD construction discussed around Section~\ref{sec:cetd} uses a separate two-state instance with biased behavior probabilities to expose the centered emphatic positivity failure analytically; we keep the two roles distinct in the text and report only the Sutton task in the algorithm-comparison tables.

\paragraph{Baird's seven-state counterexample.}
Baird's counterexample uses the seven-state feature construction, zero reward, and discount \(\gamma=0.99\). The comparison includes TD, GTD2, TDC, TDRC, CETD, and RETD with \(10\) repetitions and \(10000\) update steps. ETD and TETD are not plotted because, at the main-text stepsize, their stored trajectories are numerically unusable; they are reported as \textsc{Div.} in Table~\ref{tab:divergence_rate} using \(50\) independent repetitions to make the divergence statistic stable. Because \(\gamma=0.99\) makes the conservative sufficient bound \(c\geq\gamma/(1-\gamma)=99\), the Baird experiment should be read as an empirical fixed-stepsize stress test rather than as an invocation of the bound.

\paragraph{New two-state prediction task.}
The new two-state task has two states and two actions. The target policy always selects action~1; the behavior policy is uniform. From state~0, action~0 keeps the process in state~0 with reward~0 and action~1 moves to state~1 with reward~0; from state~1, action~0 moves to state~0 with reward~0 and action~1 stays in state~1 with reward~1. We use \(\gamma=0.9\) and one-dimensional scalar features \(\phi(s_1)=1\), \(\phi(s_2)=0.6\), i.e.\ \(\bm{\Phi}=[1,\,0.6]^\top\). The comparison includes TD, GTD2, TDC, TDRC, ETD, TETD, CETD, and RETD with \(10\) repetitions and \(5000\) update steps. This task is the controlled CETD-vs-RETD setting reported in the main text.

\paragraph{TD key-matrix analysis: why TD diverges on the Sutton two-state task but converges on the new two-state task.}
Both tasks use one-dimensional scalar features, so the off-policy TD key matrix
\(\bm{A}_{\mathrm{TD}}=\bm{\Phi}^\top\bm{D}_\mu(\bm{I}-\gamma\bm{P}_\pi)\bm{\Phi}\)
is a single scalar. Its sign decides everything: if \(\bm{A}_{\mathrm{TD}}>0\), the mean-field ODE \(\dot\theta=b_{\mathrm{TD}}-\bm{A}_{\mathrm{TD}}\theta\) is stable; if \(\bm{A}_{\mathrm{TD}}<0\), it is unstable.

For both tasks, \(\bm{P}_\pi=\bigl[\!\begin{smallmatrix}0&1\\0&1\end{smallmatrix}\!\bigr]\), \(\bm{D}_\mu=\bigl[\!\begin{smallmatrix}0.5&0\\0&0.5\end{smallmatrix}\!\bigr]\), and \(\gamma=0.9\), so
\(\bm{D}_\mu(\bm{I}-\gamma\bm{P}_\pi)=\bigl[\!\begin{smallmatrix}0.5&-0.45\\0&0.05\end{smallmatrix}\!\bigr]\). The two tasks differ only in the scalar feature vector \(\bm{\Phi}\).

\emph{Sutton two-state task.} \(\bm{\Phi}=[1,\,2]^\top\):
\begin{equation*}
\bm{A}_{\mathrm{TD}}
=
\begin{bmatrix}1&2\end{bmatrix}
\begin{bmatrix}0.5&-0.45\\0&0.05\end{bmatrix}
\begin{bmatrix}1\\2\end{bmatrix}
=\begin{bmatrix}0.5&-0.35\end{bmatrix}\begin{bmatrix}1\\2\end{bmatrix}
=0.5-0.7=-0.2.
\end{equation*}
\(\bm{A}_{\mathrm{TD}}=-0.2<0\), so the ODE is unstable: \(\dot\theta=0.2\,\theta\) and any nonzero initialization grows exponentially. Simulation matches: at \(\alpha=0.01\), \(\theta_0=1\), \(5000\) steps, \(|\theta|\) reaches order \(10^4\), consistent with the divergence statistic in Table~\ref{tab:divergence_rate} and the unbounded TD curve in Figure~\ref{fig:appendix_alpha_0p01}(e). This is the classical structural instability that motivated emphatic and gradient-TD methods.

\emph{New two-state task.} \(\bm{\Phi}=[1,\,0.6]^\top\):
\begin{equation*}
\bm{A}_{\mathrm{TD}}
=
\begin{bmatrix}1&0.6\end{bmatrix}
\begin{bmatrix}0.5&-0.45\\0&0.05\end{bmatrix}
\begin{bmatrix}1\\0.6\end{bmatrix}
=\begin{bmatrix}0.5&-0.42\end{bmatrix}\begin{bmatrix}1\\0.6\end{bmatrix}
=0.5-0.252=0.248.
\end{equation*}
\(\bm{A}_{\mathrm{TD}}=0.248>0\), so the ODE \(\dot\theta=b_{\mathrm{TD}}-0.248\,\theta\) is stable with a unique equilibrium \(\theta^\star=b_{\mathrm{TD}}/\bm{A}_{\mathrm{TD}}=0.3/0.248\approx 1.21\). Because the scalar feature \(\bm{\Phi}=[1,0.6]^\top\) cannot represent \(\bm{v}_\pi=[9,10]^\top\), the asymptotic RMSE of TD is the projection error \(\sqrt{\tfrac12\sum_s(\phi(s)\theta^\star-v_\pi(s))^2}\approx 8.56\), not zero. Simulation at \(\alpha=0.01\) confirms this: TD trajectories remain bounded and converge to this projected fixed point, in agreement with the tail RMSE \(8.573\) in Table~\ref{tab:tail_rmse} and the smooth TD curve in Figure~\ref{fig:appendix_alpha_0p01}(g).

\emph{What the contrast isolates.} The two tasks share the same behavior policy, target policy, transition kernel, and discount factor; they differ only in the scalar feature value at state~2 (\(2\) vs.\ \(0.6\)) and in reward. Yet \(\bm{A}_{\mathrm{TD}}\) flips sign: \(-0.2\) on the Sutton task, \(+0.248\) on the new task. Off-policy TD divergence on the Sutton task is therefore a \emph{mean-field} phenomenon visible already at the sign of \(\bm{A}_{\mathrm{TD}}\), not a stochastic-noise artifact, and is exactly what classical emphatic and gradient-TD methods are designed to repair. The new two-state task removes this mean-field obstacle by construction (\(\bm{A}_{\mathrm{TD}}>0\)): TD itself is stable, so the comparison between CETD and RETD reported in the main text and in Tables~\ref{tab:tail_rmse}--\ref{tab:divergence_rate} isolates the effect of regularizing the auxiliary centering channel rather than confounding it with TD-style off-policy divergence. This is also why CETD's failure on this task (tail RMSE \(43.73\), max RMSE \(64.92\), markedly worse than the \(\approx 8.5\) projection-limited baselines) is informative: it cannot be blamed on TD-level instability and must come from the centered emphatic coupling itself, exactly the failure mode that RETD repairs.

\paragraph{ETD, CETD, and RETD key matrices on the same two two-state tasks.}
We extend the TD-level analysis to the three emphatic methods. ETD has a scalar key matrix \(\bm{A}_{\mathrm{ETD}}=\bm{\Phi}^\top\bm{F}(\bm{I}-\gamma\bm{P}_\pi)\bm{\Phi}\), where \(\bm{f}=(\bm{I}-\gamma\bm{P}_\pi^\top)^{-1}\bm{d}_\mu\) is the emphatic distribution and \(\bm{F}=\mathrm{diag}(\bm{f})\). CETD and RETD instead track an auxiliary scalar \(\omega\), so their key matrices live in \(\mathbb{R}^{2\times 2}\):
\begin{equation*}
\overline{\bm{G}}_{\mathrm{CETD}}
=\begin{bmatrix}
\bm{A}_{\mathrm{ETD}} & \bm{\Phi}^\top\bm{d}_\mu\\
\bm{d}_\mu^\top\bm{\Phi} & 1
\end{bmatrix},
\qquad
\overline{\bm{G}}_{\mathrm{RETD}}
=\begin{bmatrix}
\bm{A}_{\mathrm{ETD}} & \bm{\Phi}^\top\bm{d}_\mu\\
\bm{d}_\mu^\top\bm{\Phi} & 1+c
\end{bmatrix}.
\end{equation*}
Both are symmetric (using the emphatic identity \(\bm{f}^\top(\bm{I}-\gamma\bm{P}_\pi)=\bm{d}_\mu^\top\), see Appendix~\ref{app:convergence_proof}); positive definiteness is therefore a leading-principal-minor check, and \(\overline{\bm{G}}_{\mathrm{RETD}}\succ 0\) is equivalent to \(-\overline{\bm{G}}_{\mathrm{RETD}}\) being Hurwitz, i.e.\ to mean-field ODE stability. Both tasks share \(\bm{P}_\mu=\bigl[\!\begin{smallmatrix}0.5&0.5\\0.5&0.5\end{smallmatrix}\!\bigr]\), \(\bm{P}_\pi=\bigl[\!\begin{smallmatrix}0&1\\0&1\end{smallmatrix}\!\bigr]\), \(\gamma=0.9\), so \(\bm{d}_\mu=[0.5,0.5]^\top\) and \(\bm{f}=[0.5,9.5]^\top\). The only thing that changes across tasks is \(\bm{\Phi}\), which feeds three structural quantities: the ETD key \(\bm{A}_{\mathrm{ETD}}\) (top-left of \(\overline{\bm{G}}\)), the off-diagonal coupling \(\bm{\Phi}^\top\bm{d}_\mu\), and through them \(\det\overline{\bm{G}}_{\mathrm{CETD}}=\bm{A}_{\mathrm{ETD}}-(\bm{\Phi}^\top\bm{d}_\mu)^2\) and the threshold
\(c_{\min}=(\bm{\Phi}^\top\bm{d}_\mu)^2/\bm{A}_{\mathrm{ETD}}-1\) above which RETD becomes PD. Table~\ref{tab:two_state_keymatrices} reports all five quantities on both tasks together with the corresponding empirical outcome at the experimental \(\alpha=0.01\) and \(c=9\).

\begin{table}[t]
\centering
\footnotesize
\caption{Key-matrix structure and empirical outcome of TD, ETD, CETD, and RETD on the two two-state tasks. Both tasks use uniform behavior \(\bm{d}_\mu=[0.5,0.5]^\top\), target \(\bm{P}_\pi=\bigl[\!\begin{smallmatrix}0&1\\0&1\end{smallmatrix}\!\bigr]\), \(\gamma=0.9\), and emphatic weighting \(\bm{f}=[0.5,9.5]^\top\); they differ only in the scalar feature \(\bm{\Phi}\) at state~2 and in reward. \textbf{Mean-field stable?} marks ODE stability (TD: sign of \(\bm{A}_{\mathrm{TD}}\); ETD: sign of \(\bm{A}_{\mathrm{ETD}}\); CETD/RETD: \(\overline{\bm{G}}\succ 0\)). RETD uses the experimental \(c=9\); \(c_{\min}\) is the smallest \(c\) for which \(\overline{\bm{G}}_{\mathrm{RETD}}\succ 0\). Empirical statistics are over \(10\) seeds, \(5000\) steps; \emph{Div.\ }denotes at least one diverged seed. Values matched against Tables~\ref{tab:tail_rmse}--\ref{tab:divergence_rate}.}
\label{tab:two_state_keymatrices}
\setlength{\tabcolsep}{3.0pt}
\resizebox{\linewidth}{!}{%
\begin{tabular}{@{}llccccc@{}}
\toprule
Task & Method & Key matrix & \(\det\) and eigenvalues & \(c_{\min}\) & Tail RMSE & Max RMSE \\
\midrule
\multirow{4}{*}{\shortstack[l]{Sutton\\ \(\bm{\Phi}=[1,2]^\top\)}}
 & TD   & \(\bm{A}_{\mathrm{TD}}=-0.2\)  & \(-0.2\) (unstable) & --- & \textsc{Div.} & \(7.10\!\times\!10^{4}\) \\
 & ETD  & \(\bm{A}_{\mathrm{ETD}}=3.4\)   & \(3.4\) (stable)    & --- & \(1.07\!\times\!10^{-19}\) & \(1.720\) \\
 & CETD & \(\bigl[\!\begin{smallmatrix}3.4&1.5\\1.5&1\end{smallmatrix}\!\bigr]\)  & \(\det=1.15\); \(\{4.12,0.28\}\) (PD)  & --- & \(7.01\!\times\!10^{-5}\) & \(2.043\) \\
 & RETD\((c{=}9)\) & \(\bigl[\!\begin{smallmatrix}3.4&1.5\\1.5&10\end{smallmatrix}\!\bigr]\) & \(\det=31.75\); \(\{10.32,3.08\}\) (PD) & \(-0.338\) & \(\mathbf{1.63\!\times\!10^{-21}}\) & \(1.754\) \\
\midrule
\multirow{4}{*}{\shortstack[l]{New\\ \(\bm{\Phi}=[1,0.6]^\top\)}}
 & TD   & \(\bm{A}_{\mathrm{TD}}=+0.248\) & \(+0.248\) (stable) & --- & \(8.573\) & \(8.802\) \\
 & ETD  & \(\bm{A}_{\mathrm{ETD}}=0.572\) & \(0.572\) (stable)  & --- & \(5.015\) & \(3.89\!\times\!10^{3}\,\)\textsc{Div.} \\
 & CETD & \(\bigl[\!\begin{smallmatrix}0.572&0.8\\0.8&1\end{smallmatrix}\!\bigr]\) & \(\det=-0.068\); \(\{1.61,-0.04\}\) (\textbf{not PD}) & --- & \(43.73\) & \(64.92\) \\
 & RETD\((c{=}9)\) & \(\bigl[\!\begin{smallmatrix}0.572&0.8\\0.8&10\end{smallmatrix}\!\bigr]\) & \(\det=5.08\); \(\{10.07,0.50\}\) (PD) & \(\mathbf{+0.119}\) & \(\mathbf{4.131}\) & \(\mathbf{8.758}\) \\
\bottomrule
\end{tabular}%
}
\end{table}

The table tells the structural story in three lines. \emph{First}, the emphatic weighting \(\bm{F}=\mathrm{diag}(0.5,9.5)\) is what flips Sutton's TD key from \(-0.2\) to an ETD key of \(+3.4\): heavy emphasis on the absorbing state \(s_2\) repairs the off-policy mean-field instability that breaks TD, and the \(2\times 2\) extensions \(\overline{\bm{G}}_{\mathrm{CETD}}\) and \(\overline{\bm{G}}_{\mathrm{RETD}}\) inherit a comfortably positive top-left, large enough that the off-diagonal \(1.5\) cannot drag the determinant negative (\(\det=1.15>0\), and any \(c\geq 0\) keeps RETD PD since \(c_{\min}=-0.338\)). All three emphatic methods therefore converge on Sutton's task while TD diverges. \emph{Second}, the same emphatic weighting that helps Sutton produces the new failure on the new two-state task: the smaller feature gap \(\bm{\Phi}=[1,0.6]^\top\) drops the ETD key to a much smaller \(0.572\) while the off-diagonal only mildly shrinks to \(0.8\), so \((\bm{\Phi}^\top\bm{d}_\mu)^2=0.64\) overpowers \(\bm{A}_{\mathrm{ETD}}=0.572\) and \(\det\overline{\bm{G}}_{\mathrm{CETD}}=-0.068<0\). The eigenvalues \(\{1.61,-0.04\}\) certify that \(-\overline{\bm{G}}_{\mathrm{CETD}}\) is not Hurwitz; the unstable mode along the negative eigenvector is exactly what causes the empirical CETD trajectories to drift to tail RMSE \(43.73\) and max RMSE \(64.92\). \emph{Third}, RETD repairs this by lifting the lower-right entry from \(1\) to \(1+c\), pushing the determinant to \(0.572(1+c)-0.64\); the threshold \(c_{\min}\!\approx\!0.119\) is small, so the experimental \(c=9\) sits far inside the PD region (\(\det=5.08\)), and RETD's empirical tail RMSE \(4.131\) is the lowest among all eight methods compared in Table~\ref{tab:tail_rmse}, including all unbiased and emphatic baselines. The remaining ETD anomaly on the new task (PD mean-field but \(1/10\) seed diverged with max RMSE \(3.89\!\times\!10^{3}\)) is a stochastic-variance phenomenon, not a mean-field one: the follow-on trace \(M_t\) inherits the high-variance product \(\rho_t M_{t-1}\) without any centering or regularization, and Table~\ref{tab:divergence_rate} flags this as the only emphatic-family failure on this task.

\subsection{Learning rates and regularization parameters}

The main text reports the common comparison at \(\alpha=0.01\). The stored algorithm-comparison sweeps cover multiple stepsizes: random walk, Boyan chain, and Baird use \(\{0.001,0.005,0.01,0.05,0.1\}\) where the corresponding result files are available, while the two-state counterexample and the new two-state prediction task use \(\{0.001,0.005,0.01,0.05\}\), with extra very small stepsizes for some diagnostic plots.

For RETD, the main comparison uses an environment-level rule that keeps \(c\) inside the stable region observed in the \(c\)-scan while avoiding values so large that the method visually collapses to ETD: \(c=9\) for random walk, Boyan chain, the original two-state counterexample, and the new two-state task, and \(c=0.5\) for Baird. These are practical fixed-stepsize choices and are not intended to satisfy the conservative sufficient bound in Assumption~\ref{assum:regularization}. In the code convention used throughout the paper, the auxiliary recursion contains the factor \((1+c)\omega_t\), and CETD is the special case \(c=0\).

\subsection{Statistics shown in the plots}

Each stored \texttt{.npy} file has shape \((N,T+1)\), where \(N\) is the number of independent runs and \(T\) is the number of update steps. For each algorithm and environment, the plotted line is the cross-run mean RMSE at the displayed time points and the shaded region is one empirical standard deviation. To keep the PDFs readable, each curve is uniformly subsampled to roughly 120 points; no temporal smoothing is applied. The y-axis is RMSE against the true value function in random walk, Boyan chain, the original two-state task, and the new two-state task, and against the zero target in Baird's zero-reward counterexample. The tail-average RMSE in Table~\ref{tab:tail_rmse} averages the final 10\% of training steps per run before averaging across runs. A run is counted as diverged in Table~\ref{tab:divergence_rate} if its RMSE trajectory contains NaN/Inf or its maximum RMSE exceeds \(10^3\); this conservative threshold is far above the normal RMSE range of the suite and is used only to flag numerically unusable trajectories.

\subsection{Computing resources}

All reported experiments are lightweight CPU-only NumPy simulations; no GPU or accelerator is required. The individual algorithm-comparison runs store arrays of at most a few dozen trajectories with at most \(10000+1\) time points, so a single Python process uses well below 1GB of memory in our runs. On a standard modern laptop or desktop CPU, an individual environment--algorithm--stepsize configuration completes in seconds to a few minutes depending on the horizon and number of seeds, and regenerating the full set of stored curves, tables, and parameter scans is expected to take hours rather than days. The plotting and table-generation scripts are negligible compared with the simulation time and require only the stored \texttt{.npy} arrays.

\subsection{Full coverage at the main-text stepsize}

Figure~\ref{fig:appendix_alpha_0p01} reports the complete seven-environment algorithm comparison at \(\alpha=0.01\). The main text retains only the Boyan-chain and Baird panels because they isolate the geometry and stability diagnostics; this figure shows that the qualitative ordering between RETD and the baselines is consistent across the remaining environments, so the main-text selection is not the result of cherry-picked panels.

\begin{figure}[p]
\centering
\includegraphics[width=0.62\linewidth]{shared_algorithm_legend.pdf}
\vspace{0.25em}

\begin{subfigure}[t]{0.46\linewidth}\centering\includegraphics[width=\linewidth]{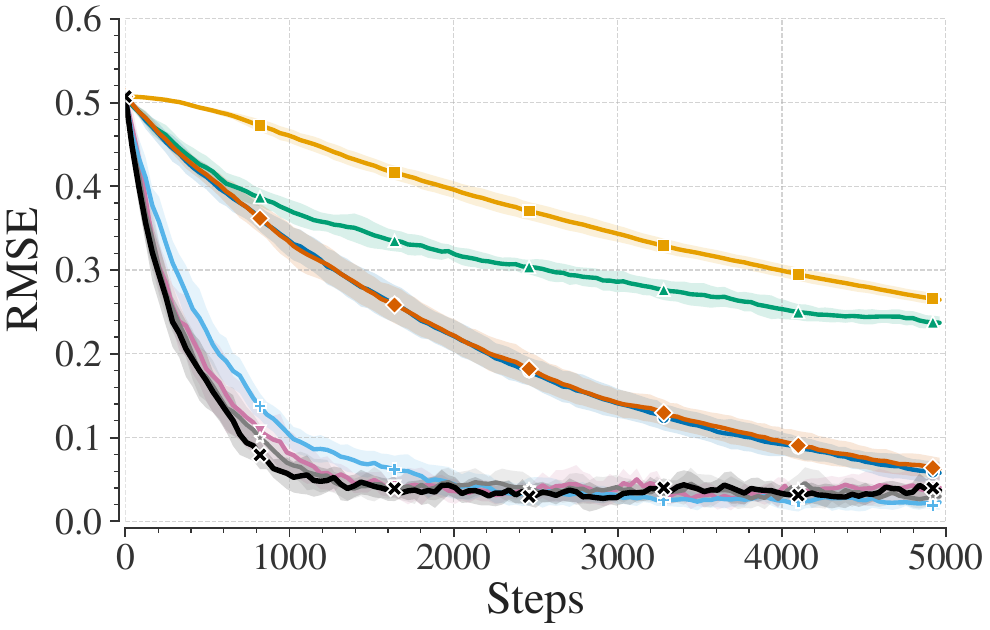}\caption{RW tabular}\end{subfigure}
\begin{subfigure}[t]{0.46\linewidth}\centering\includegraphics[width=\linewidth]{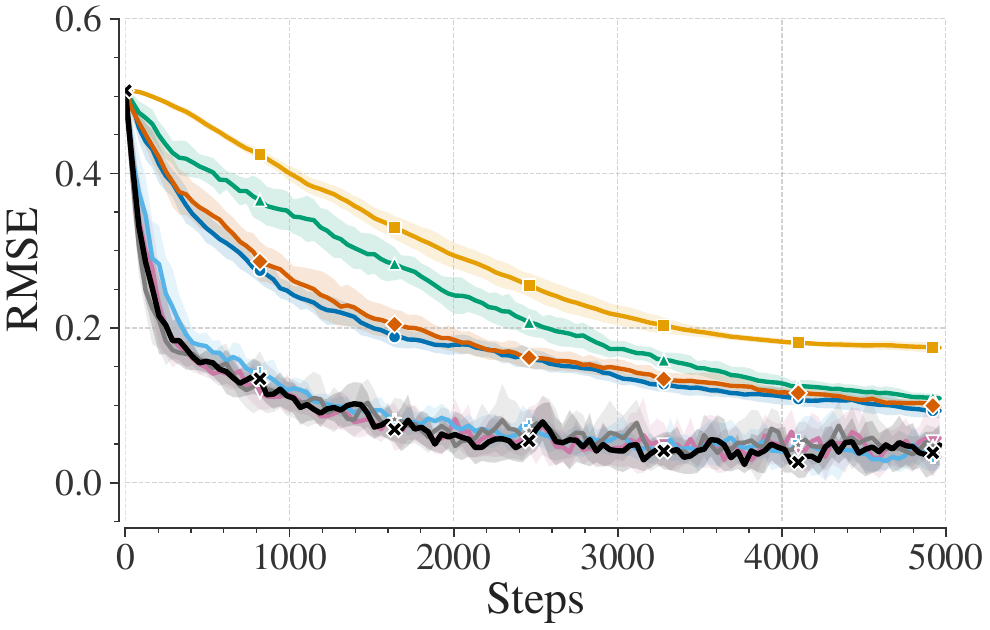}\caption{RW inverted}\end{subfigure}
\begin{subfigure}[t]{0.46\linewidth}\centering\includegraphics[width=\linewidth]{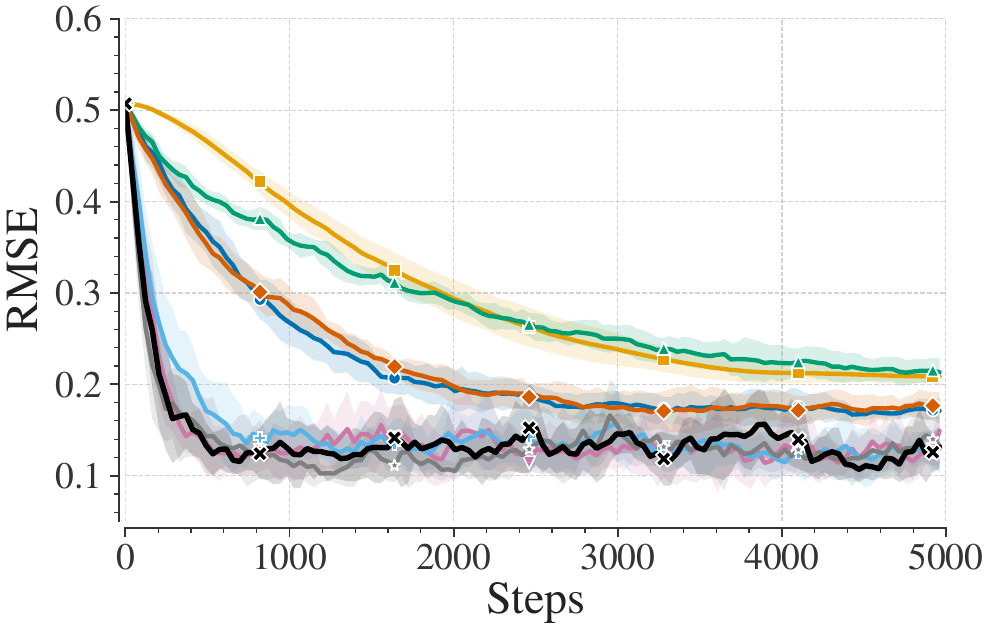}\caption{RW dependent}\end{subfigure}
\begin{subfigure}[t]{0.46\linewidth}\centering\includegraphics[width=\linewidth]{boyanchain_alpha_0p01.pdf}\caption{Boyan chain}\end{subfigure}
\begin{subfigure}[t]{0.46\linewidth}\centering\includegraphics[width=\linewidth]{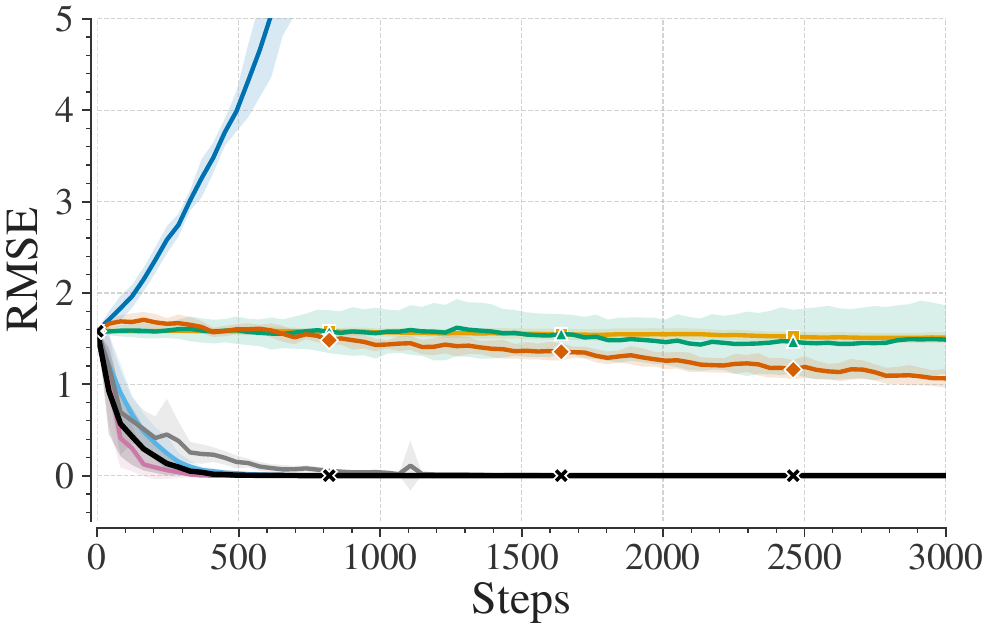}\caption{Two-state}\end{subfigure}
\begin{subfigure}[t]{0.46\linewidth}\centering\includegraphics[width=\linewidth]{baird_7state_alpha_0p01.pdf}\caption{Baird}\end{subfigure}
\begin{subfigure}[t]{0.46\linewidth}\centering\includegraphics[width=\linewidth]{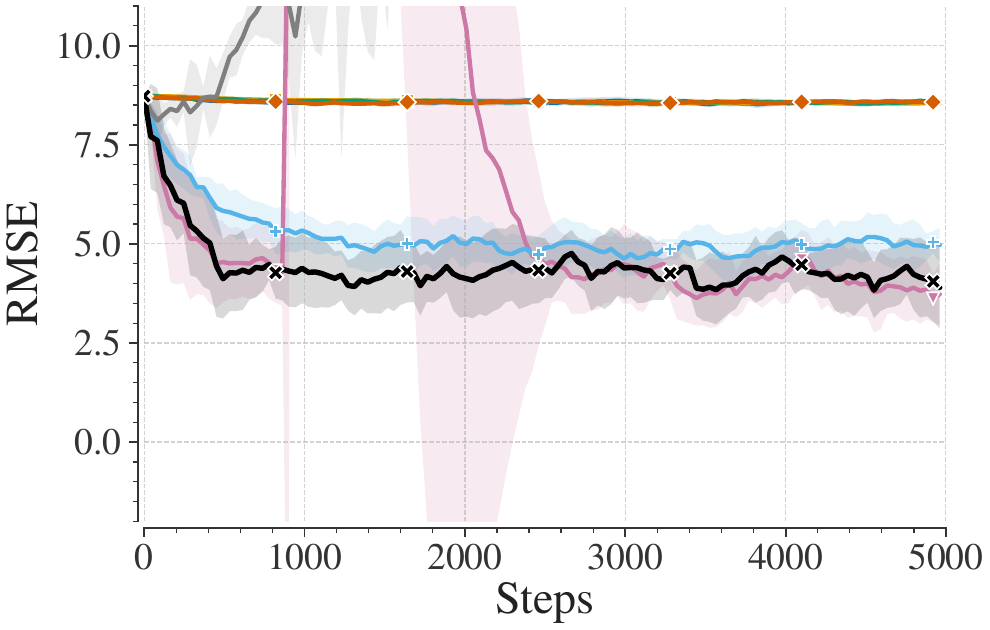}\caption{New two-state}\end{subfigure}
\caption{Complete algorithm comparisons at \(\alpha=0.01\) across all seven environments. The main text retains only the Boyan-chain and Baird panels; the remaining panels are reported here so that the main-text selection can be verified against the full coverage.}
\label{fig:appendix_alpha_0p01}
\end{figure}

\subsection{Additional learning-rate slices}

Figures~\ref{fig:appendix_alpha_0p005} and~\ref{fig:appendix_alpha_0p05} report two additional learning-rate slices, one smaller and one larger than the main-text value. They are intended to verify that the conclusions at \(\alpha=0.01\) are not specific to a single stepsize: at \(\alpha=0.005\) the differences between methods are compressed but their ordering is preserved, while at \(\alpha=0.05\) the more aggressive setting amplifies the sensitivity of emphatic-trace methods and makes the stability gap between RETD and the naive emphatic methods more visible.

\begin{figure}[p]
\centering
\begin{subfigure}[t]{0.46\linewidth}\centering\includegraphics[width=\linewidth]{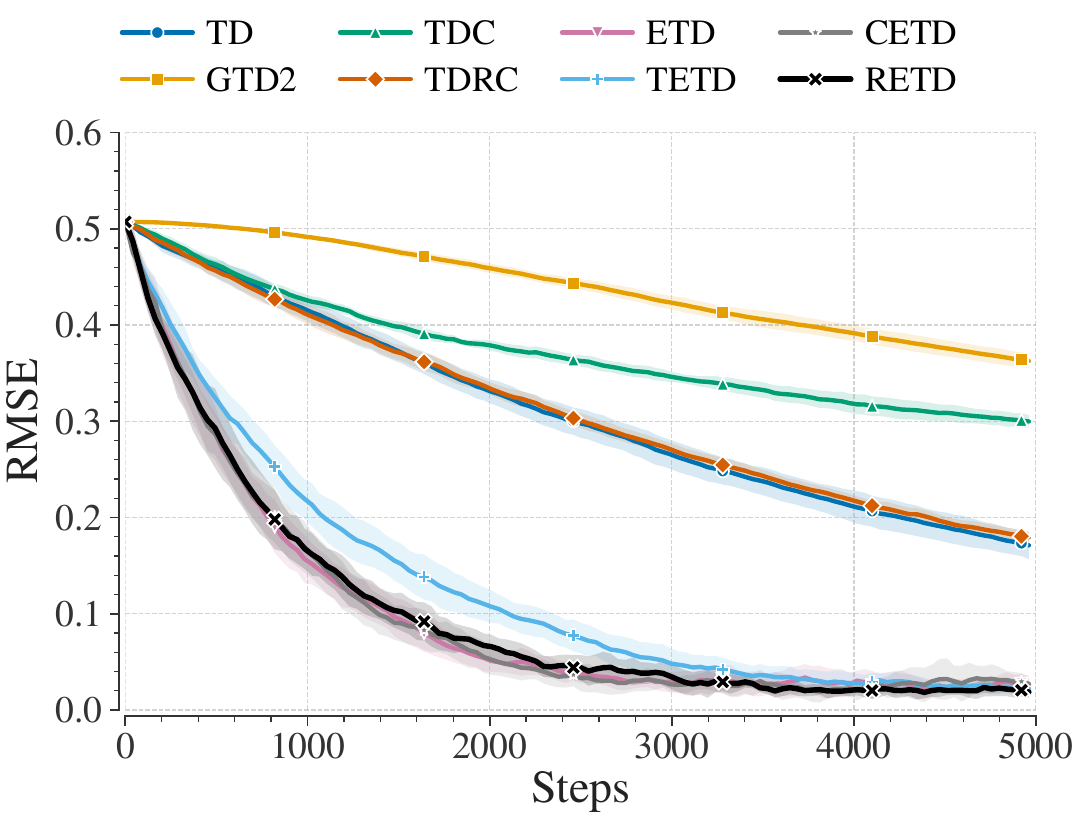}\caption{RW tabular}\end{subfigure}
\begin{subfigure}[t]{0.46\linewidth}\centering\includegraphics[width=\linewidth]{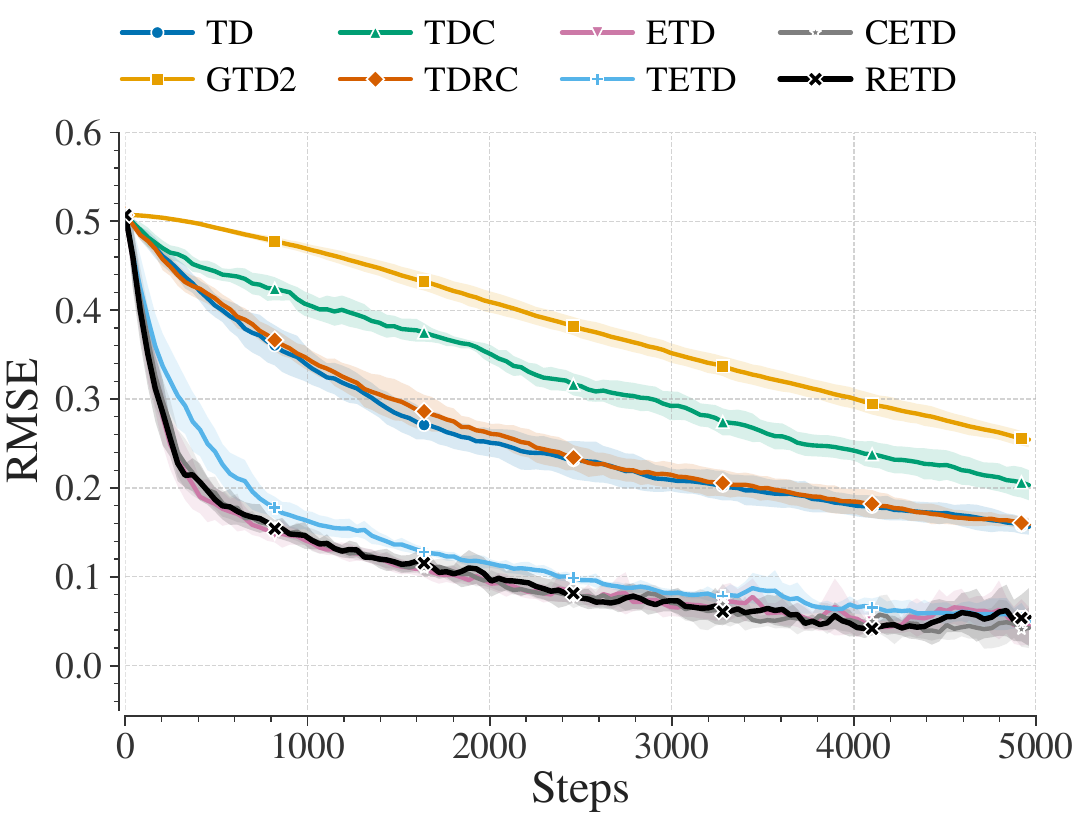}\caption{RW inverted}\end{subfigure}
\begin{subfigure}[t]{0.46\linewidth}\centering\includegraphics[width=\linewidth]{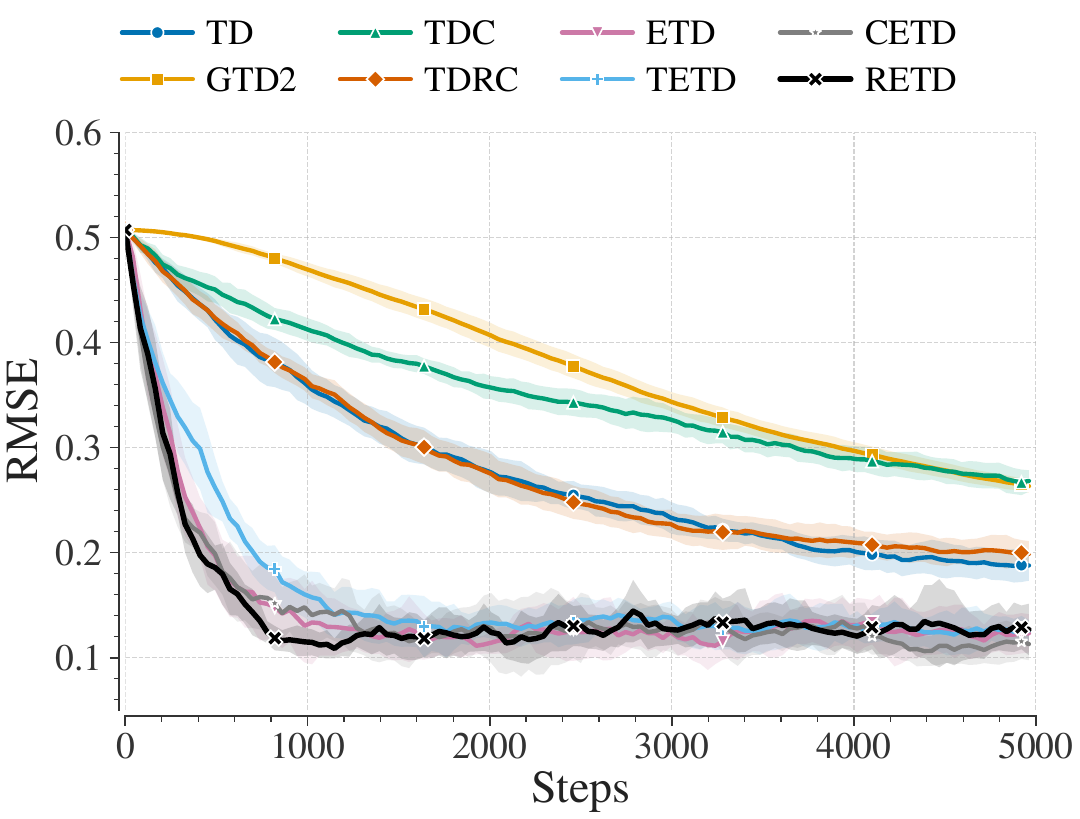}\caption{RW dependent}\end{subfigure}
\begin{subfigure}[t]{0.46\linewidth}\centering\includegraphics[width=\linewidth]{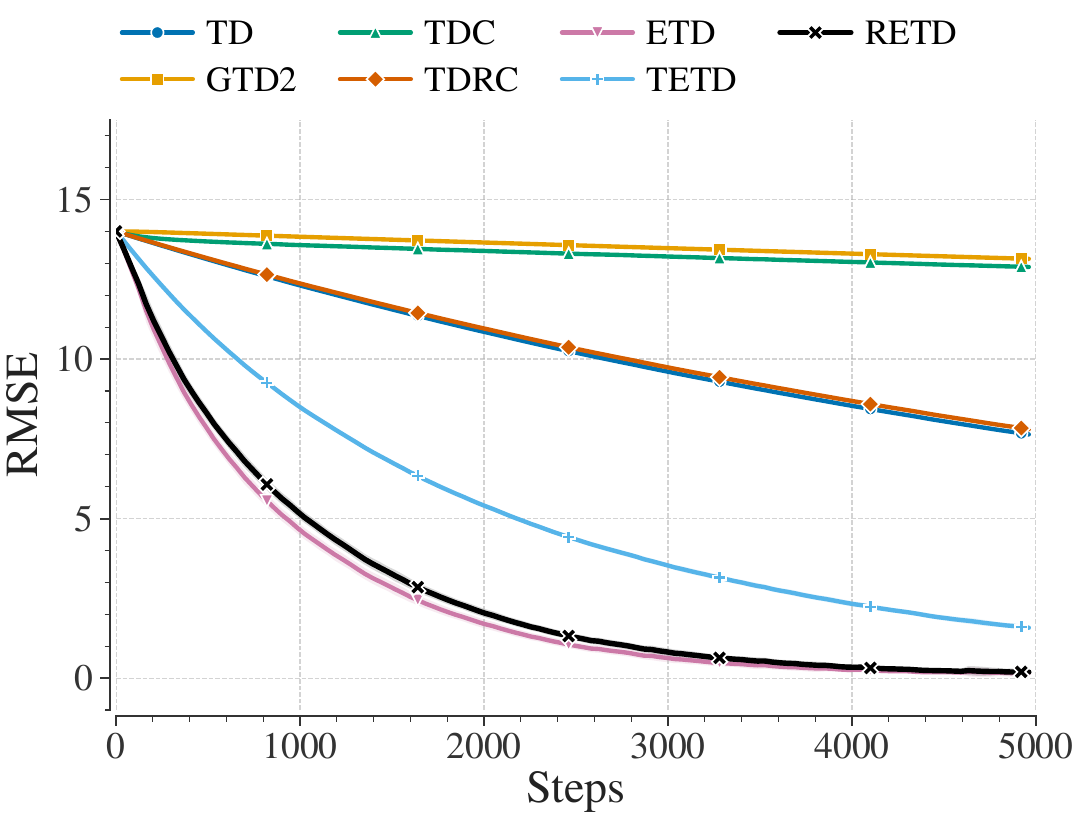}\caption{Boyan chain}\end{subfigure}
\begin{subfigure}[t]{0.46\linewidth}\centering\includegraphics[width=\linewidth]{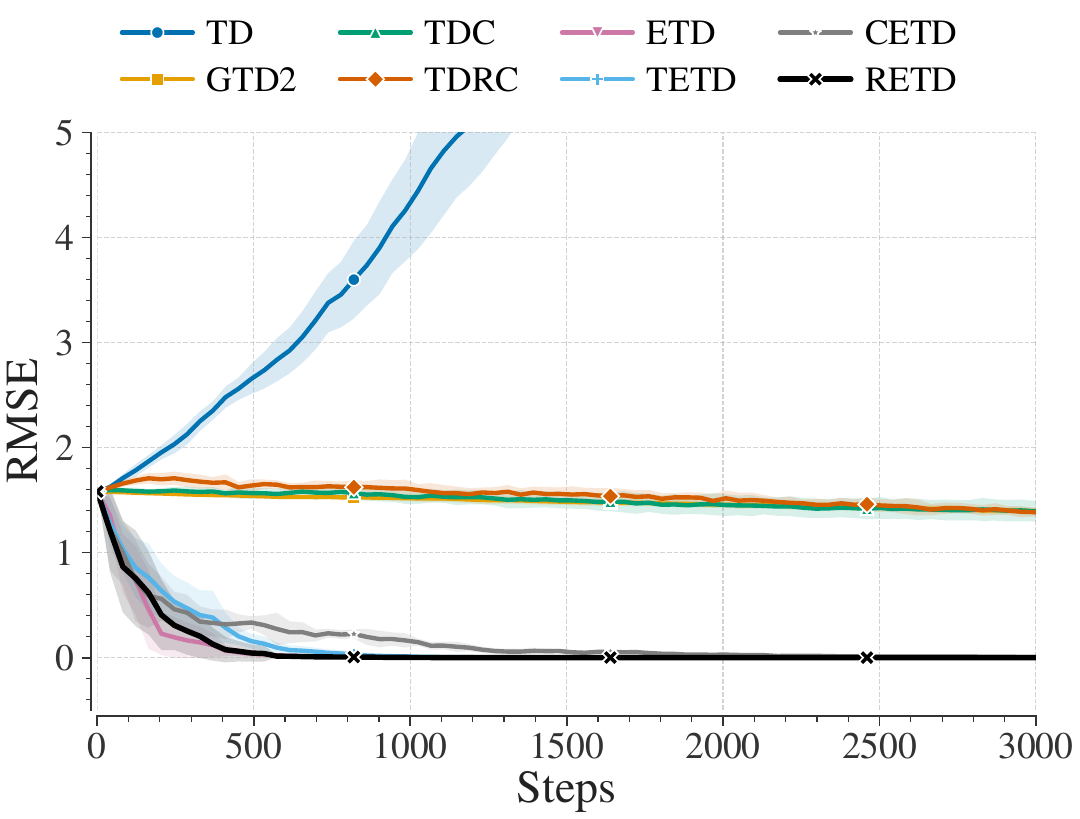}\caption{Two-state}\end{subfigure}
\begin{subfigure}[t]{0.46\linewidth}\centering\includegraphics[width=\linewidth]{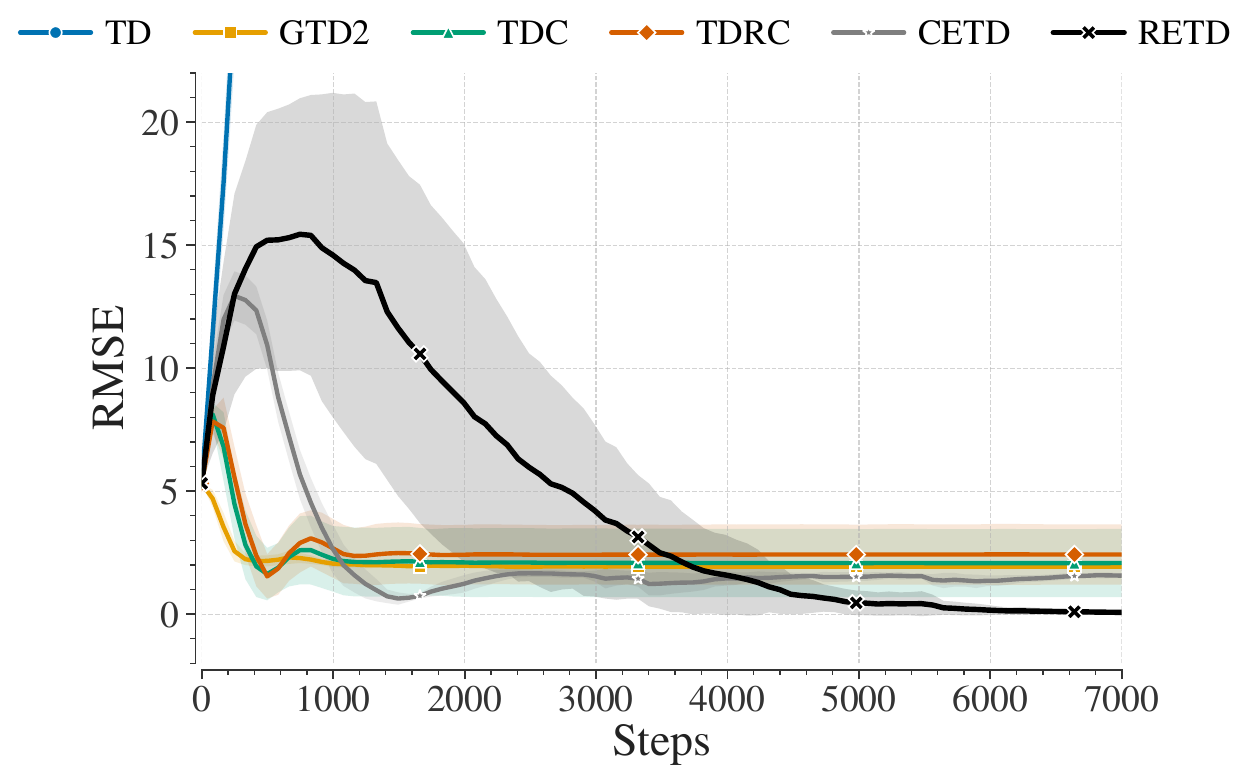}\caption{Baird}\end{subfigure}
\begin{subfigure}[t]{0.46\linewidth}\centering\includegraphics[width=\linewidth]{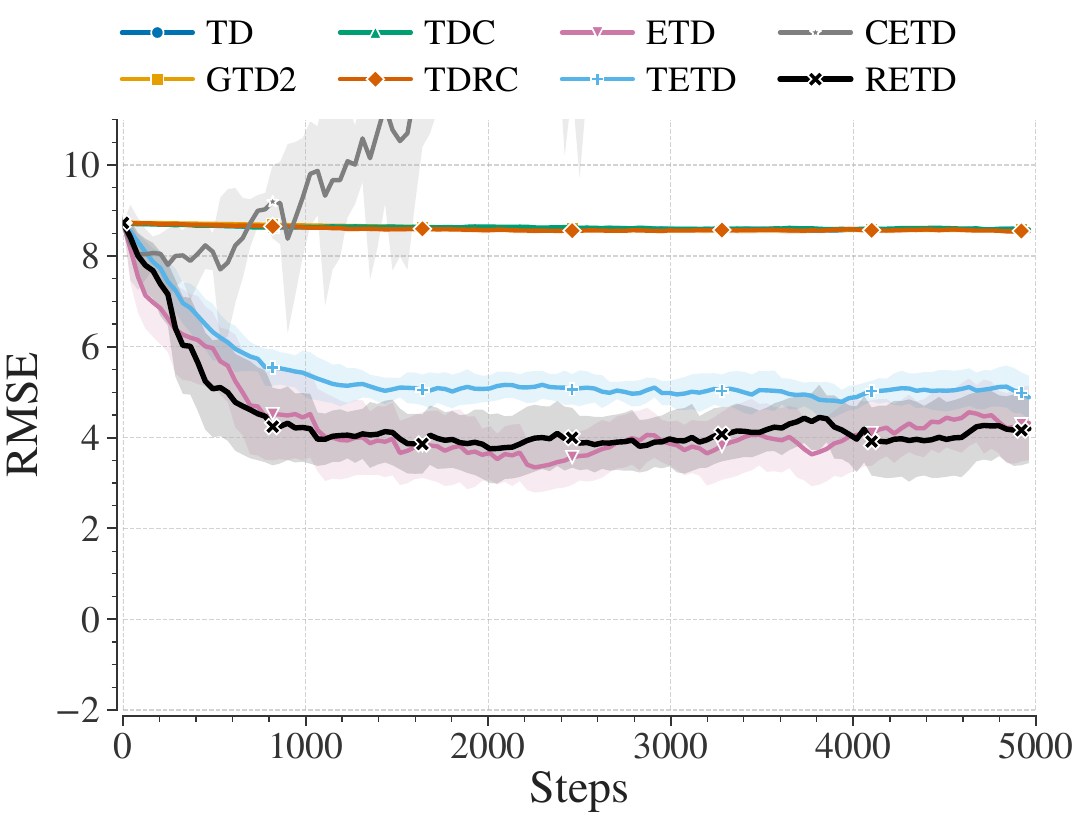}\caption{New two-state}\end{subfigure}
\caption{Algorithm comparisons at \(\alpha=0.005\). The smaller stepsize compresses the differences between methods while preserving their ordering relative to Figure~\ref{fig:appendix_alpha_0p01}.}
\label{fig:appendix_alpha_0p005}
\end{figure}

\begin{figure}[p]
\centering
\begin{subfigure}[t]{0.445\linewidth}\centering\includegraphics[width=\linewidth]{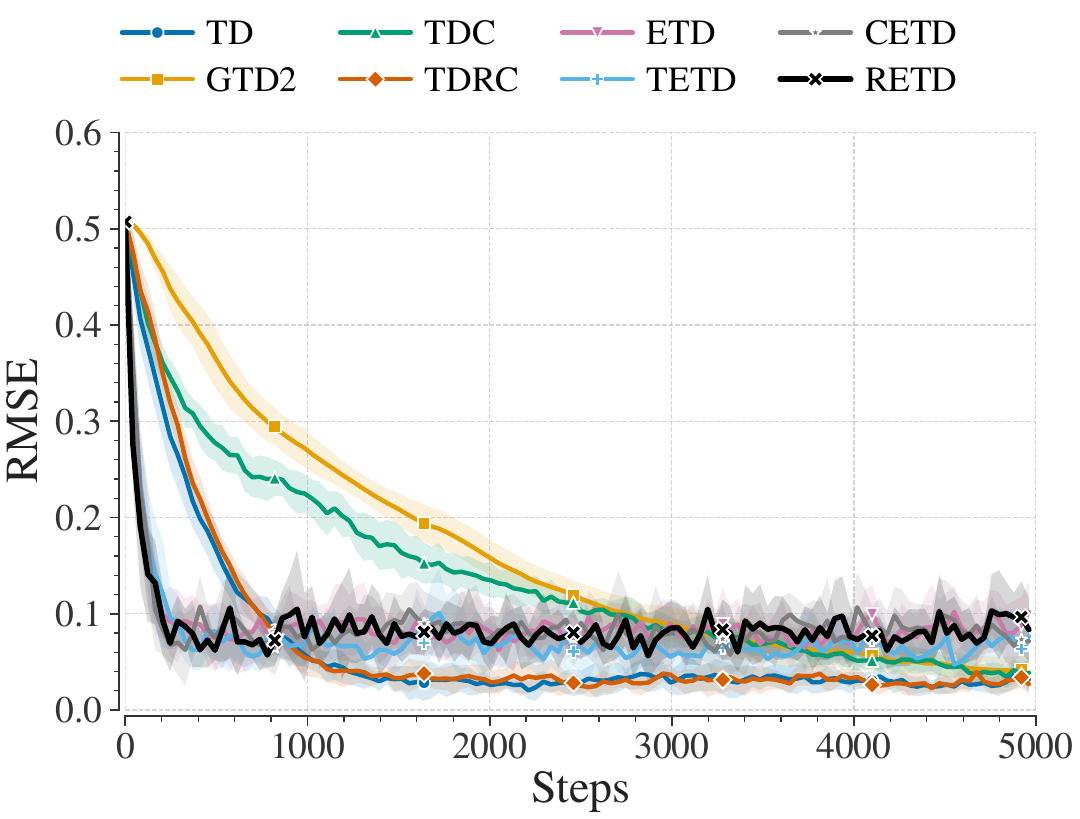}\caption{RW tabular}\end{subfigure}
\begin{subfigure}[t]{0.445\linewidth}\centering\includegraphics[width=\linewidth]{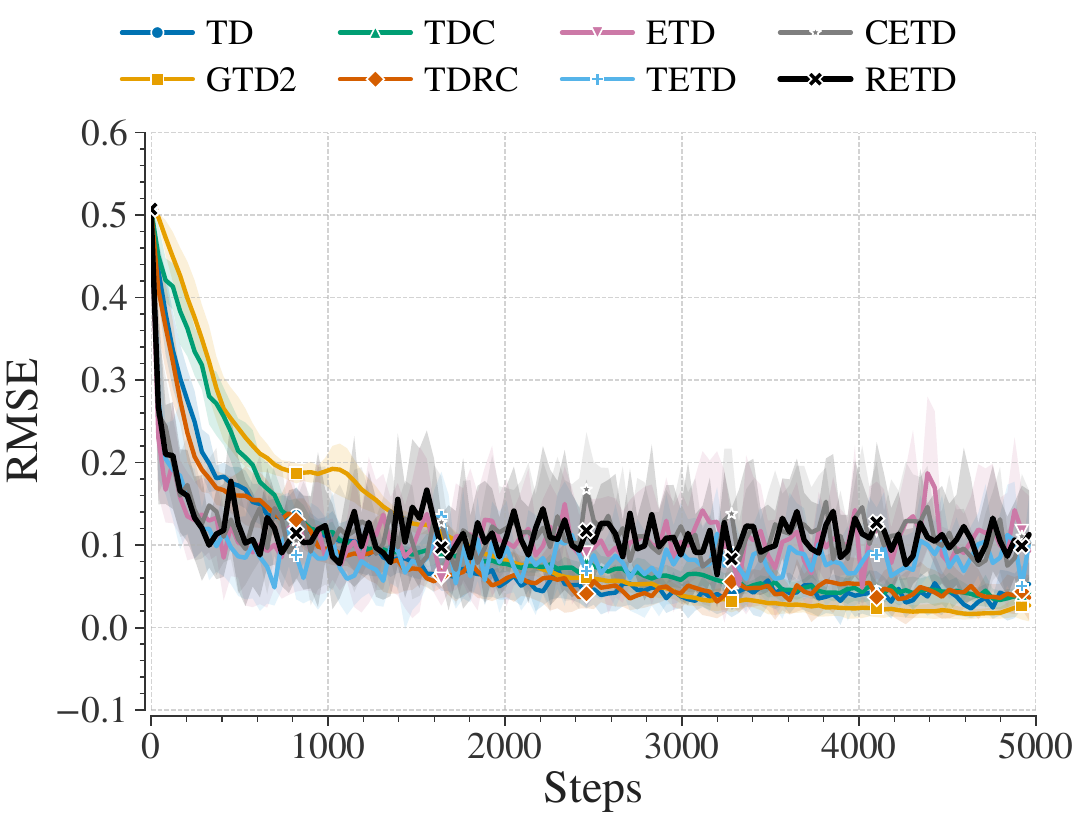}\caption{RW inverted}\end{subfigure}
\begin{subfigure}[t]{0.445\linewidth}\centering\includegraphics[width=\linewidth]{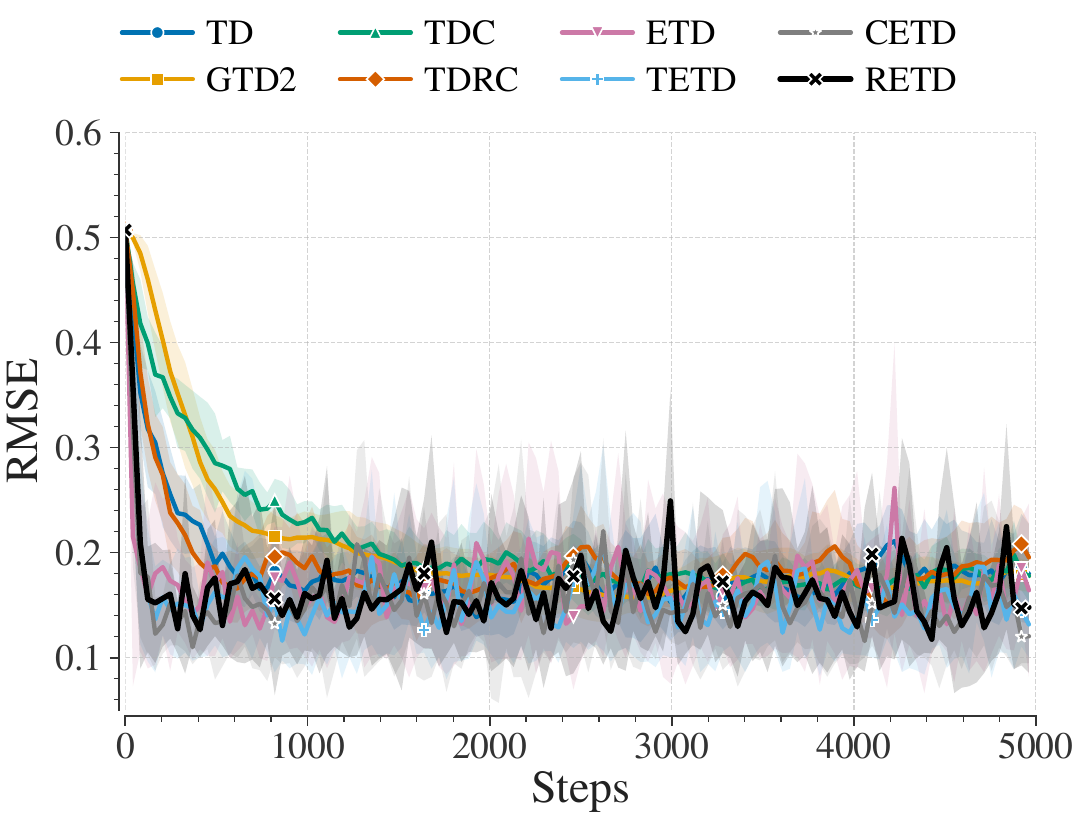}\caption{RW dependent}\end{subfigure}
\begin{subfigure}[t]{0.445\linewidth}\centering\includegraphics[width=\linewidth]{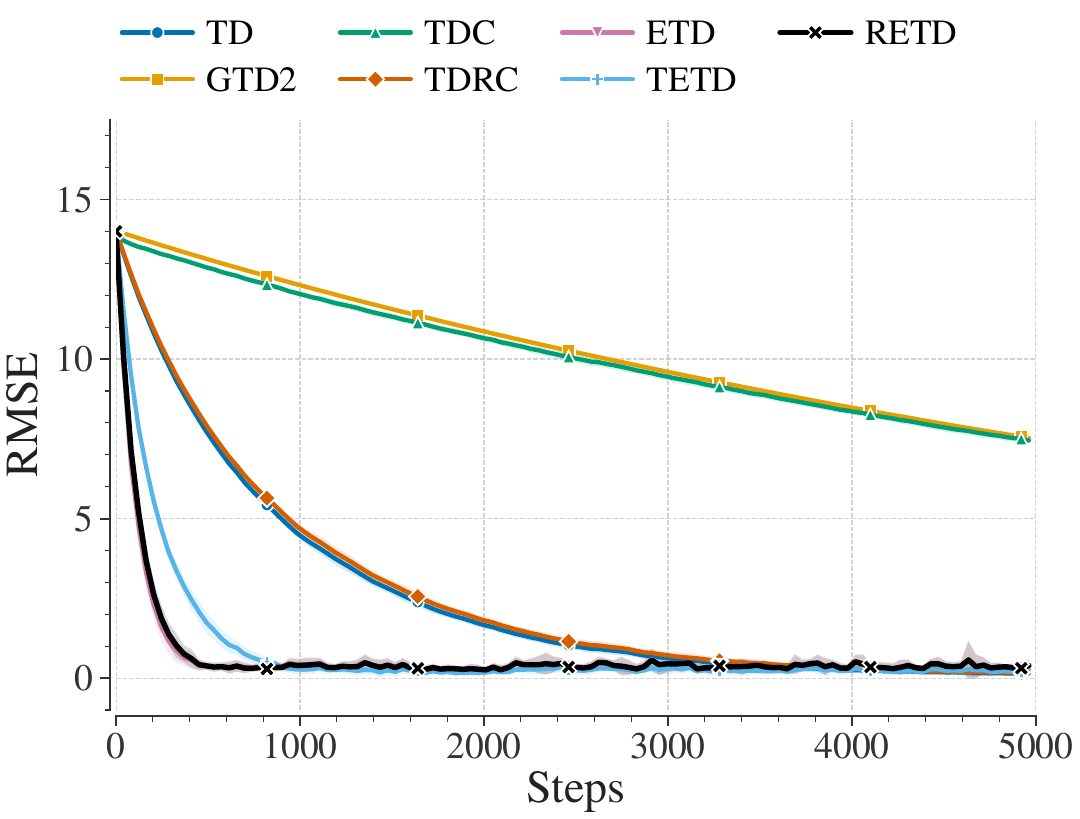}\caption{Boyan chain}\end{subfigure}
\begin{subfigure}[t]{0.445\linewidth}\centering\includegraphics[width=\linewidth]{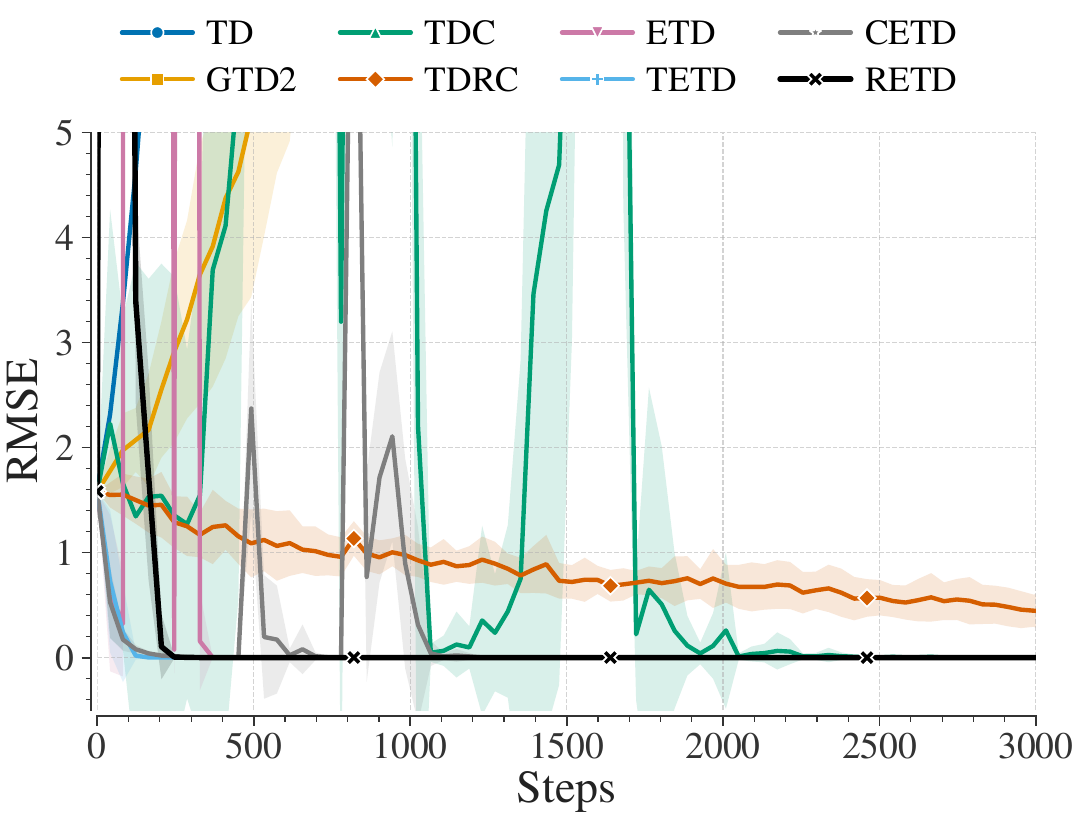}\caption{Two-state}\end{subfigure}
\begin{subfigure}[t]{0.445\linewidth}\centering\includegraphics[width=\linewidth]{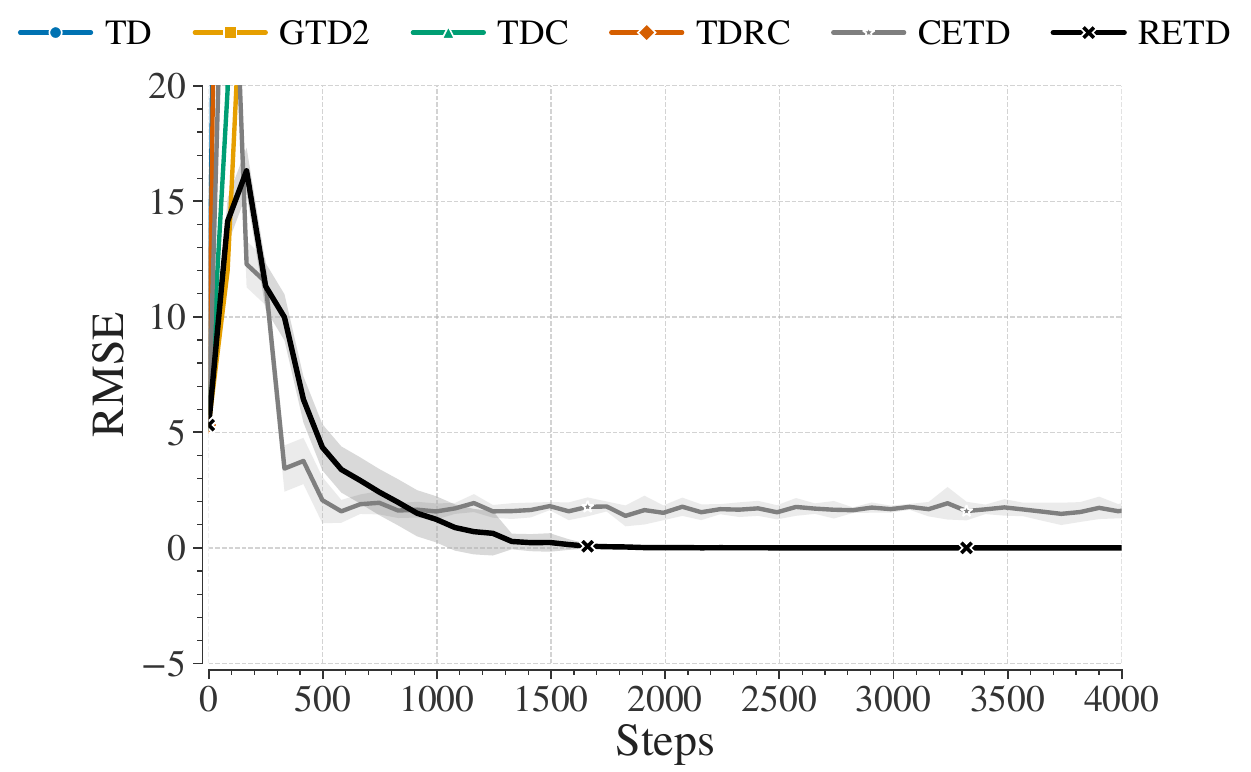}\caption{Baird}\end{subfigure}
\begin{subfigure}[t]{0.445\linewidth}\centering\includegraphics[width=\linewidth]{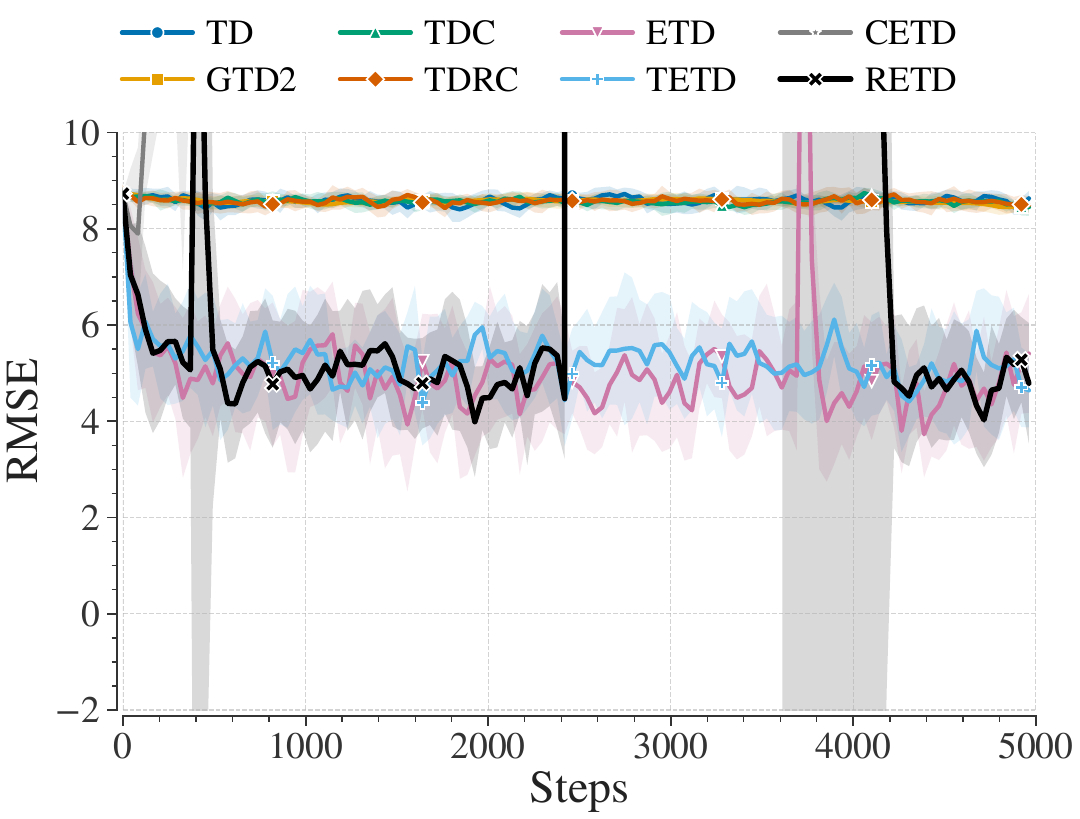}\caption{New two-state}\end{subfigure}
\caption{Algorithm comparisons at \(\alpha=0.05\). The larger stepsize amplifies the sensitivity of emphatic-trace methods and makes the stability gap between RETD and the naive emphatic methods more visible.}
\label{fig:appendix_alpha_0p05}
\end{figure}

\subsection{RETD robustness scans}

Figures~\ref{fig:appendix_c_scan_alpha_0p01} and~\ref{fig:appendix_alpha_scan_fixed_c} provide the expanded robustness evidence behind the main-text discussion of \(c\). The \(c\)-scan fixes \(\alpha=0.01\) and varies the regularization strength, illustrating the structural transition from CETD-like behavior at small \(c\) to ETD-like behavior at very large \(c\). The \(\alpha\)-scan fixes the environment-specific regularization value used in the main comparisons and varies the learning rate, showing that the method is not tuned to a single stepsize. Together they support the main-text claim that the practical regime for \(c\) is intermediate.

\begin{figure}[p]
\centering
\begin{subfigure}[t]{0.46\linewidth}\centering\includegraphics[width=\linewidth]{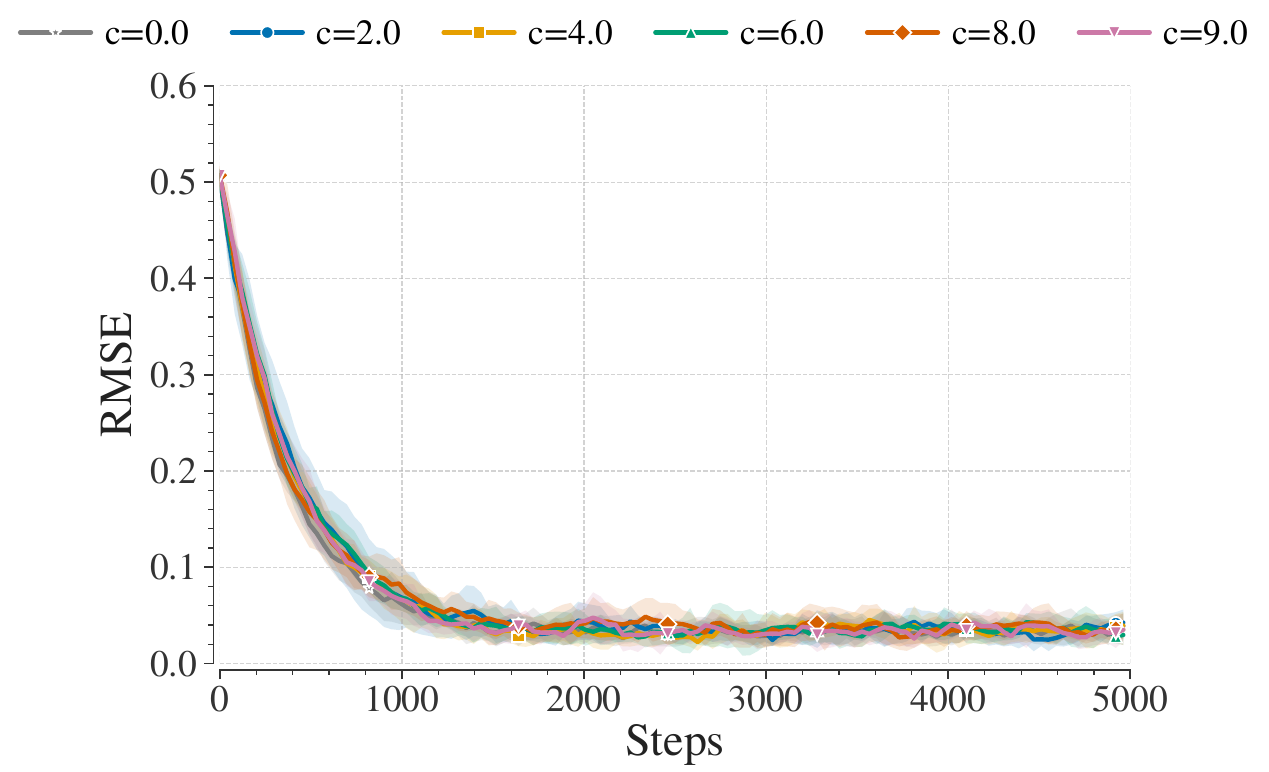}\caption{RW tabular}\end{subfigure}
\begin{subfigure}[t]{0.46\linewidth}\centering\includegraphics[width=\linewidth]{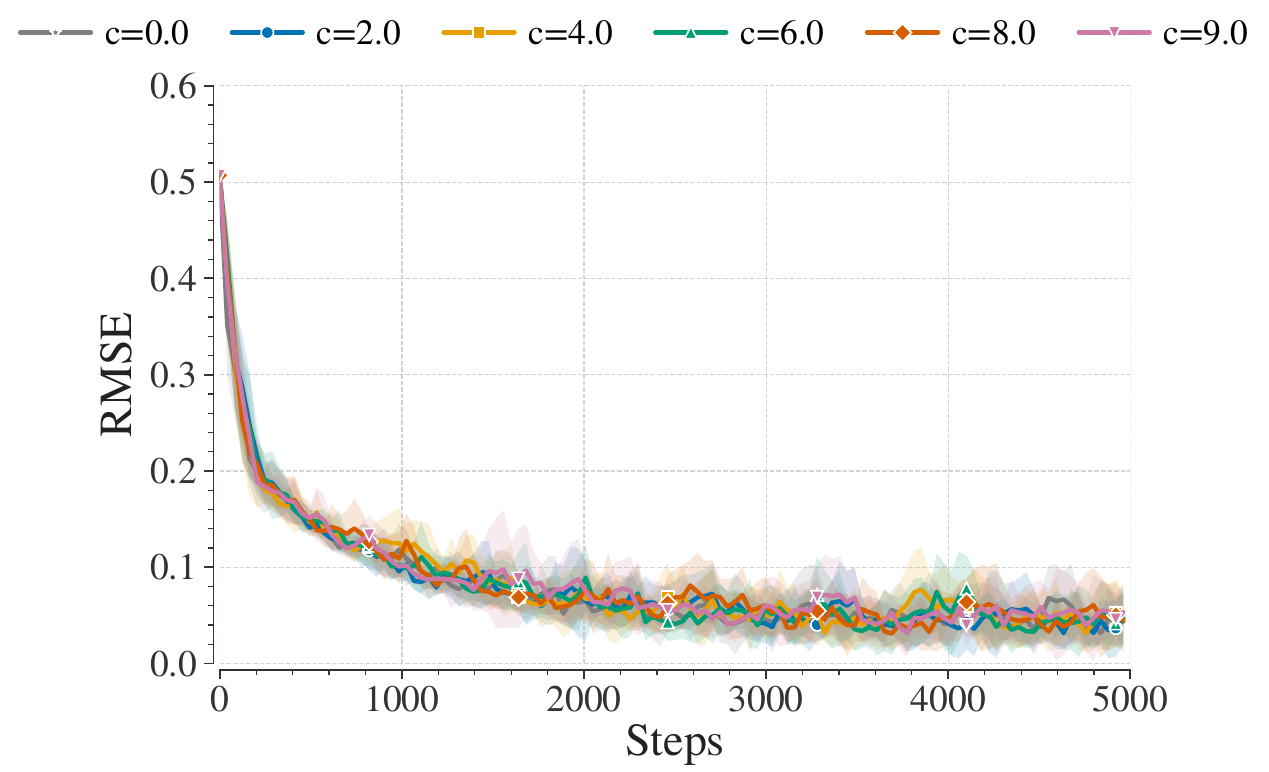}\caption{RW inverted}\end{subfigure}
\begin{subfigure}[t]{0.46\linewidth}\centering\includegraphics[width=\linewidth]{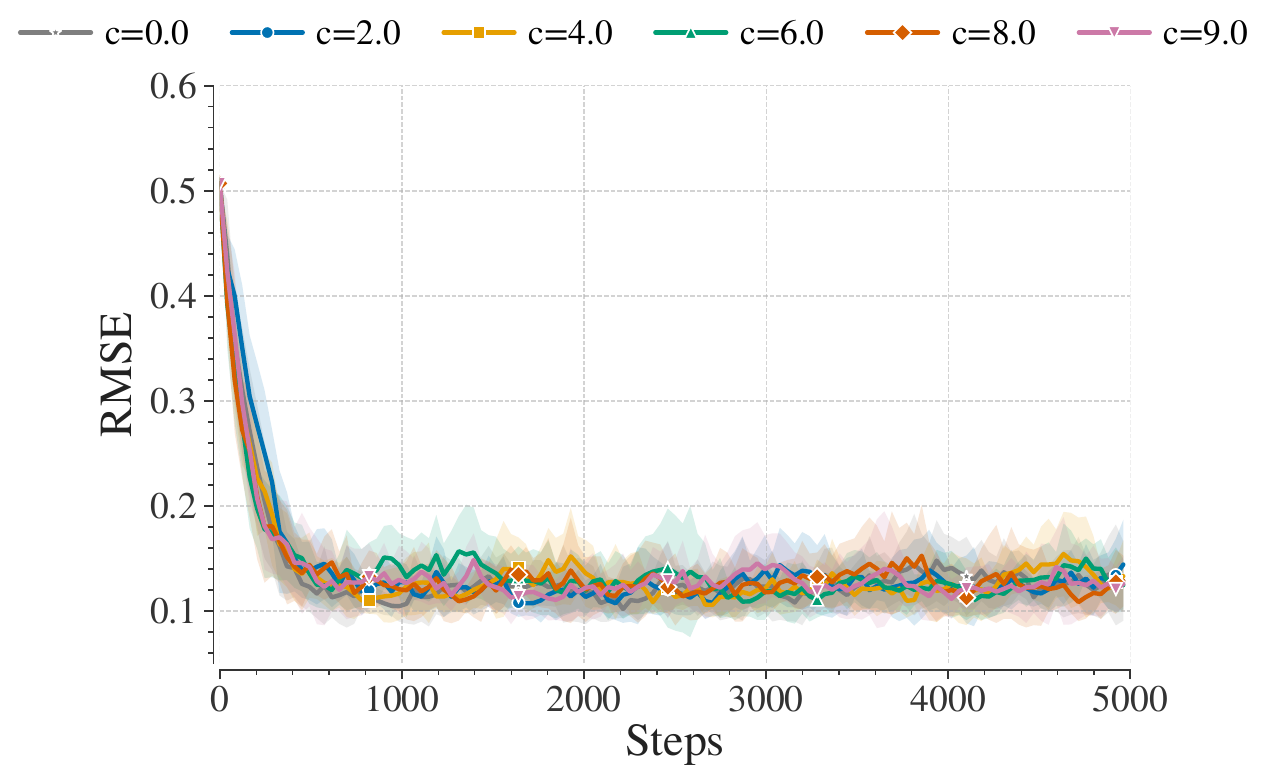}\caption{RW dependent}\end{subfigure}
\begin{subfigure}[t]{0.46\linewidth}\centering\includegraphics[width=\linewidth]{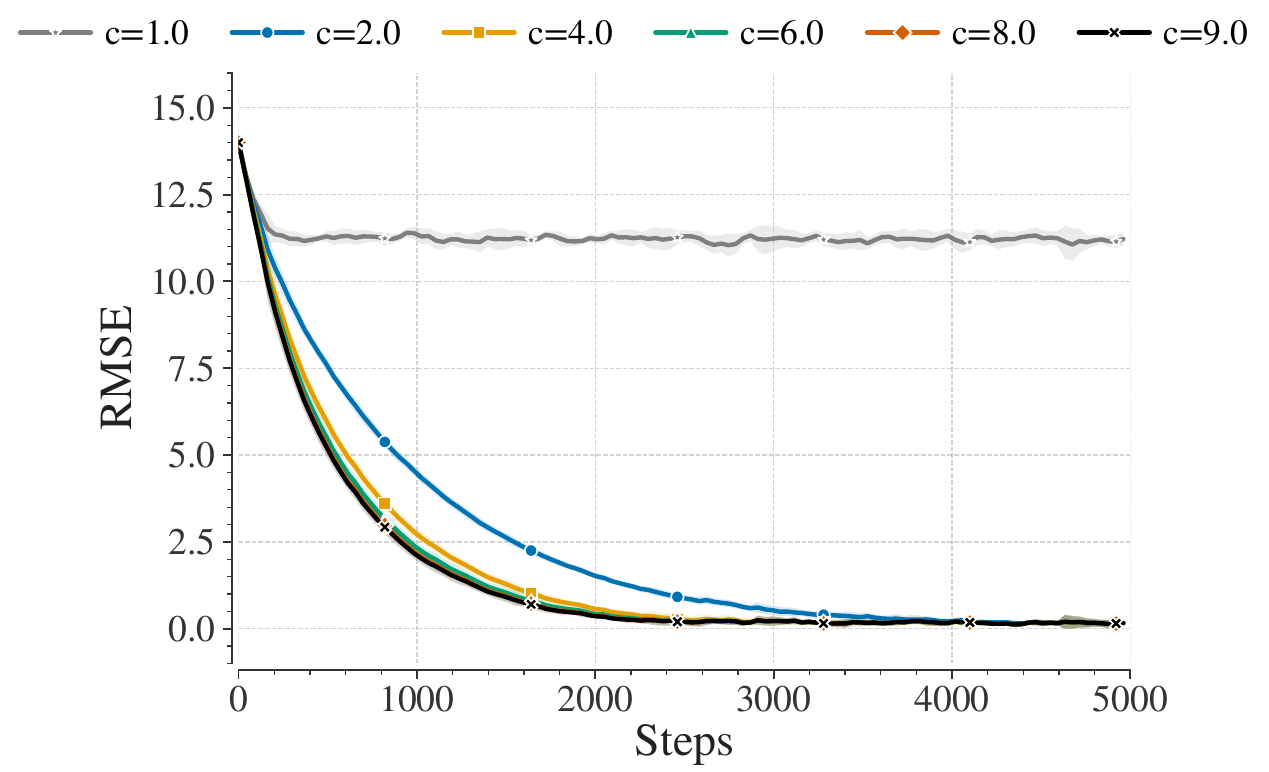}\caption{Boyan chain}\end{subfigure}
\begin{subfigure}[t]{0.46\linewidth}\centering\includegraphics[width=\linewidth]{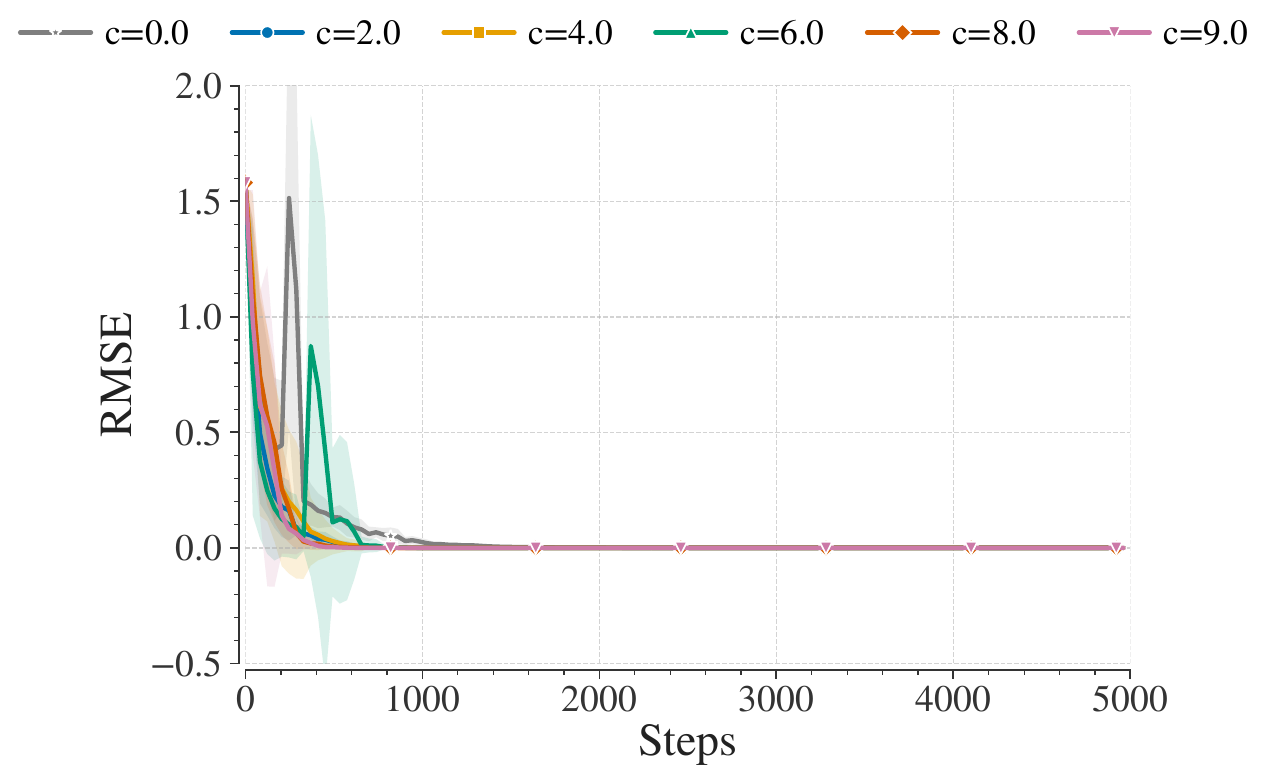}\caption{Two-state}\end{subfigure}
\begin{subfigure}[t]{0.46\linewidth}\centering\includegraphics[width=\linewidth]{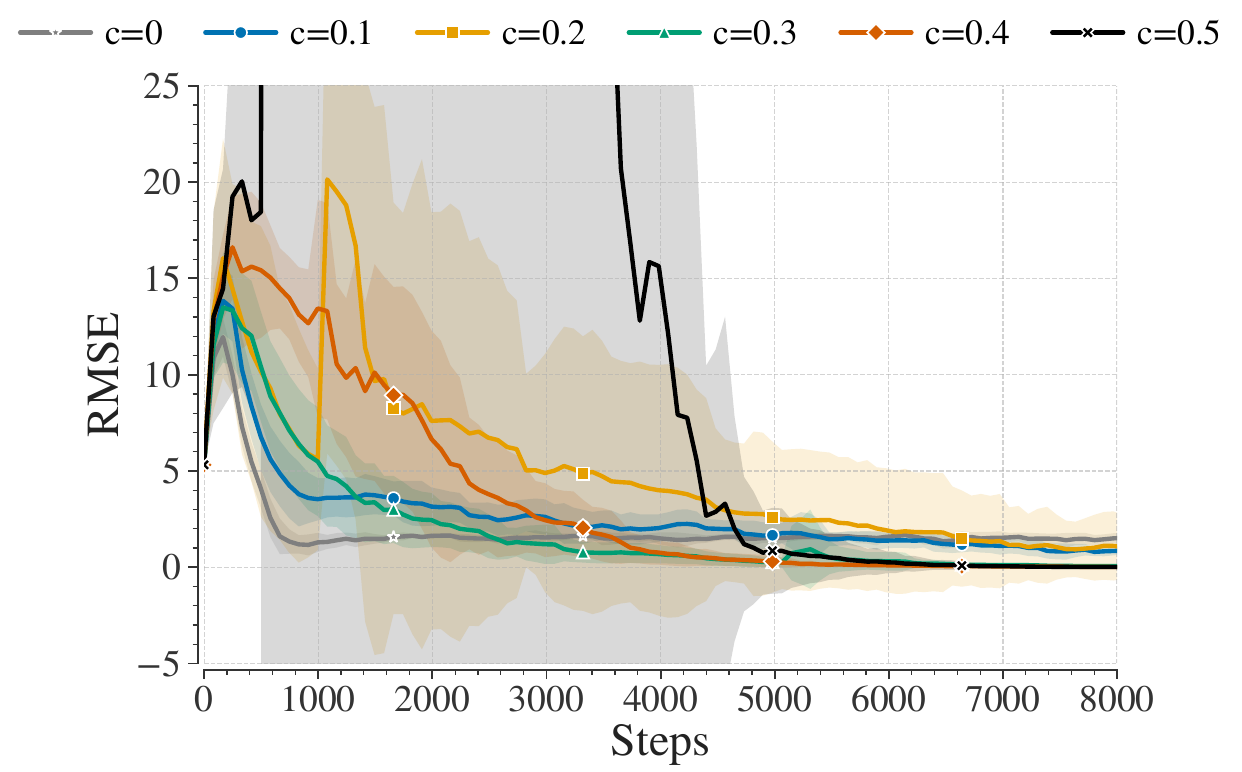}\caption{Baird}\end{subfigure}
\begin{subfigure}[t]{0.46\linewidth}\centering\includegraphics[width=\linewidth]{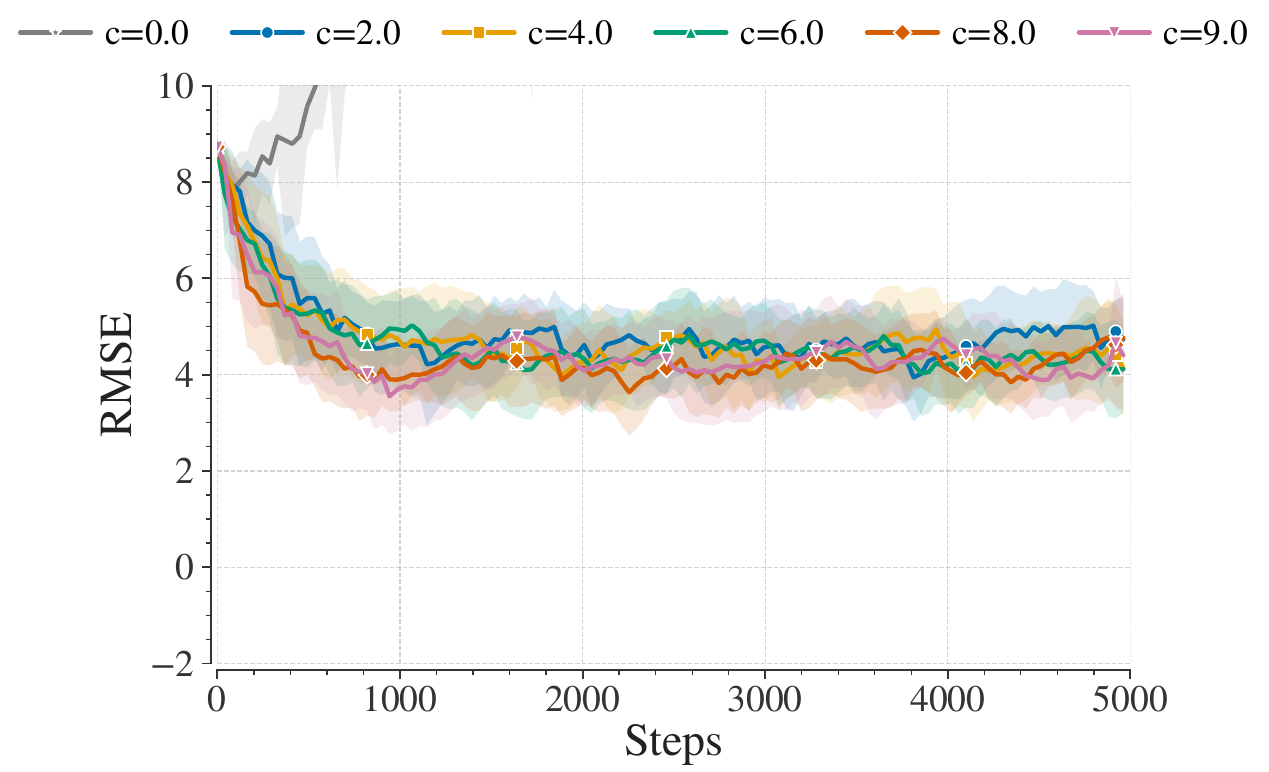}\caption{New two-state}\end{subfigure}
\caption{RETD \(c\)-scan at fixed \(\alpha=0.01\). Small \(c\) approaches CETD; very large \(c\) damps the auxiliary recursion and approaches ETD; intermediate values deliver stable centered emphatic learning.}
\label{fig:appendix_c_scan_alpha_0p01}
\end{figure}

\begin{figure}[p]
\centering
\begin{subfigure}[t]{0.46\linewidth}\centering\includegraphics[width=\linewidth]{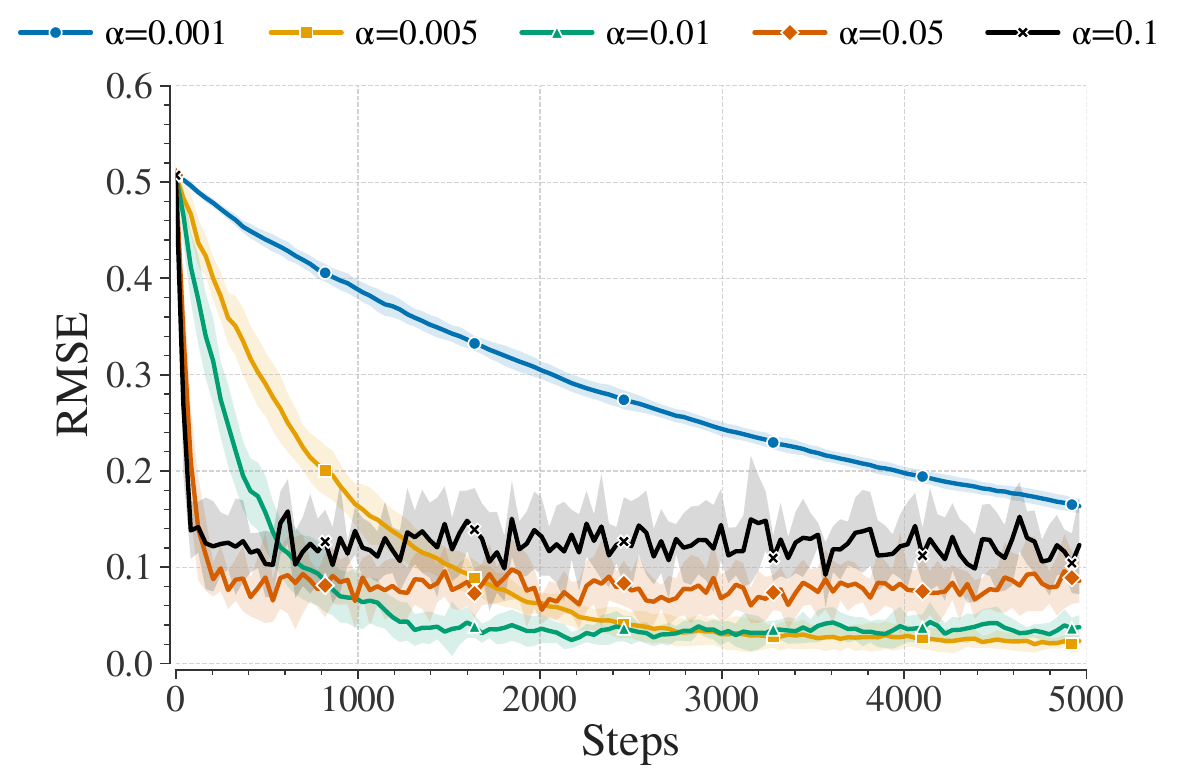}\caption{RW tabular}\end{subfigure}
\begin{subfigure}[t]{0.46\linewidth}\centering\includegraphics[width=\linewidth]{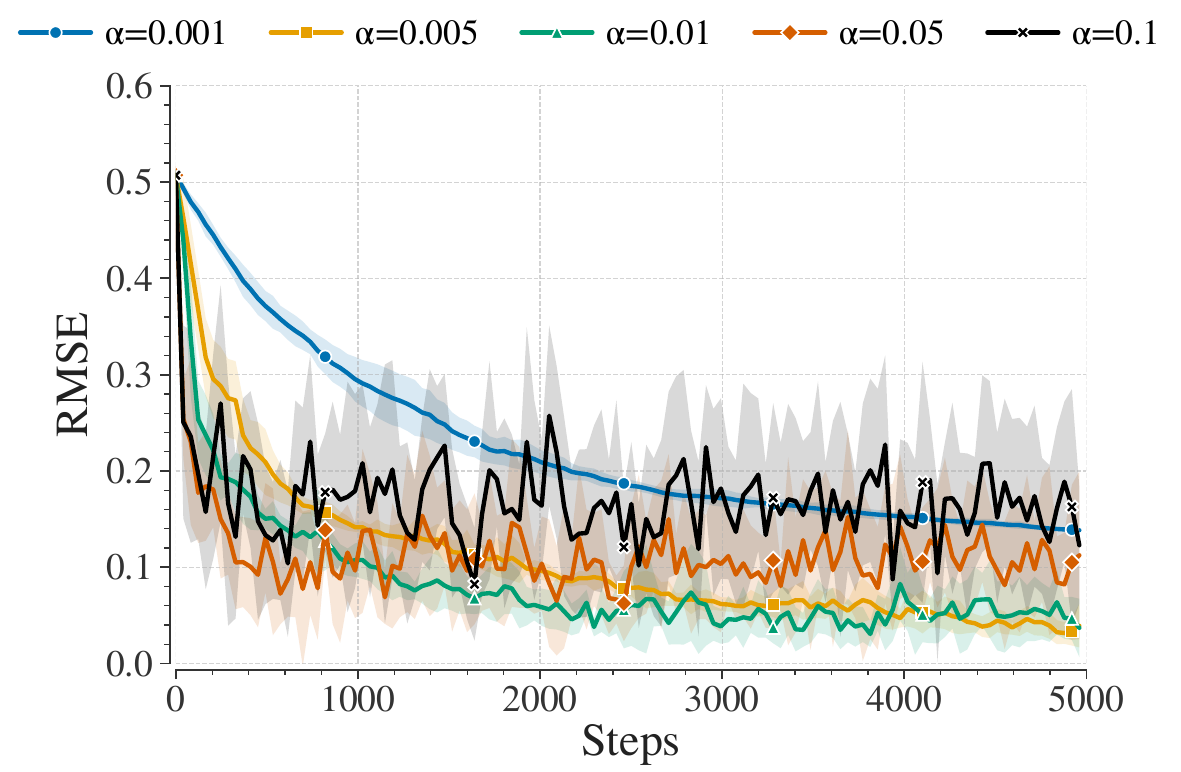}\caption{RW inverted}\end{subfigure}
\begin{subfigure}[t]{0.46\linewidth}\centering\includegraphics[width=\linewidth]{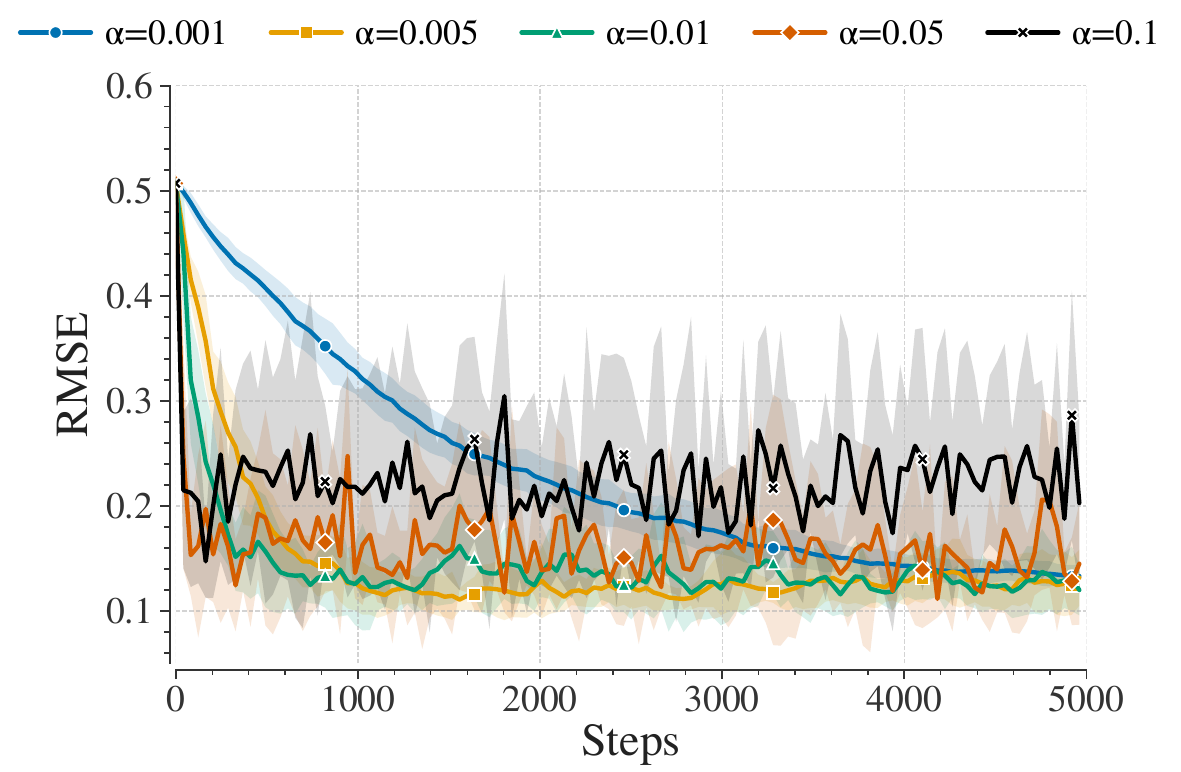}\caption{RW dependent}\end{subfigure}
\begin{subfigure}[t]{0.46\linewidth}\centering\includegraphics[width=\linewidth]{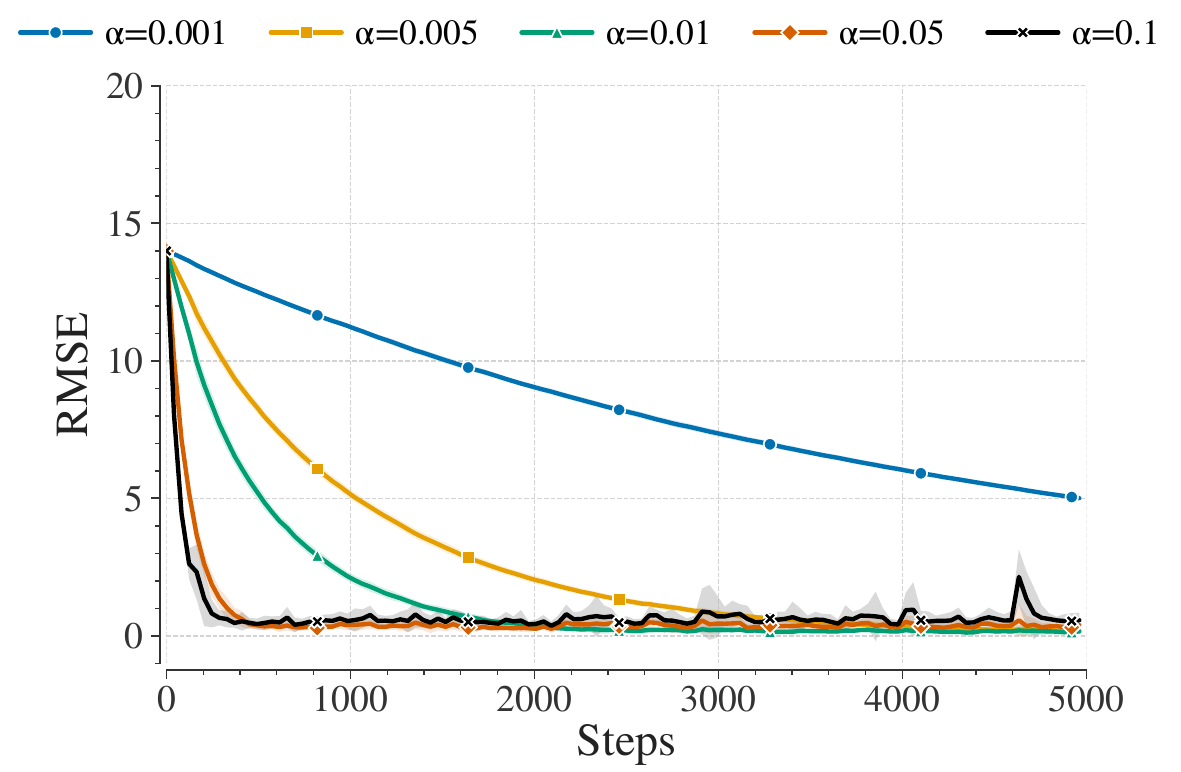}\caption{Boyan chain}\end{subfigure}
\begin{subfigure}[t]{0.46\linewidth}\centering\includegraphics[width=\linewidth]{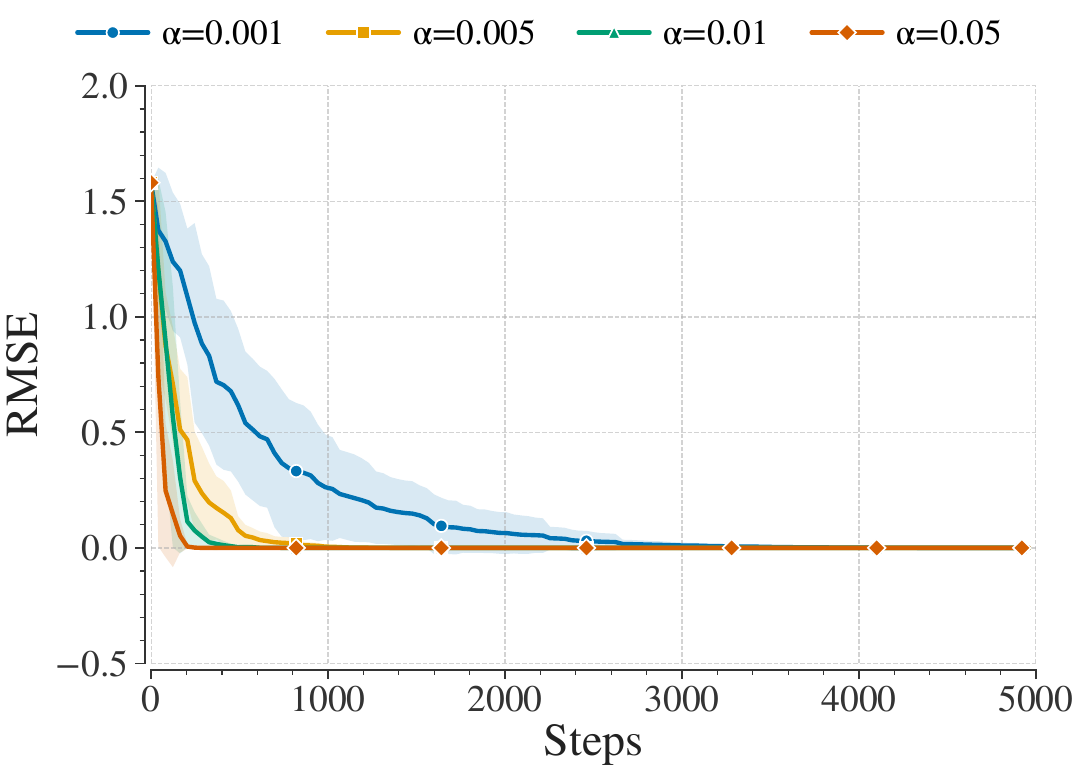}\caption{Two-state}\end{subfigure}
\begin{subfigure}[t]{0.46\linewidth}\centering\includegraphics[width=\linewidth]{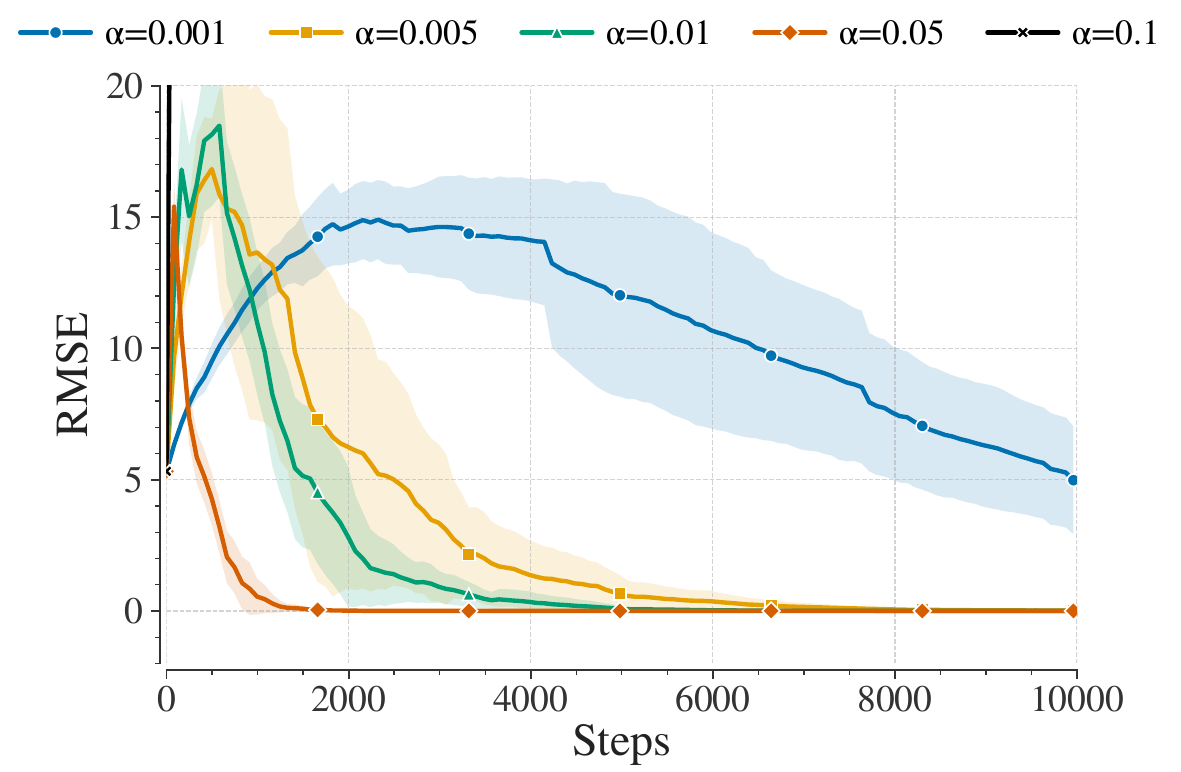}\caption{Baird}\end{subfigure}
\begin{subfigure}[t]{0.46\linewidth}\centering\includegraphics[width=\linewidth]{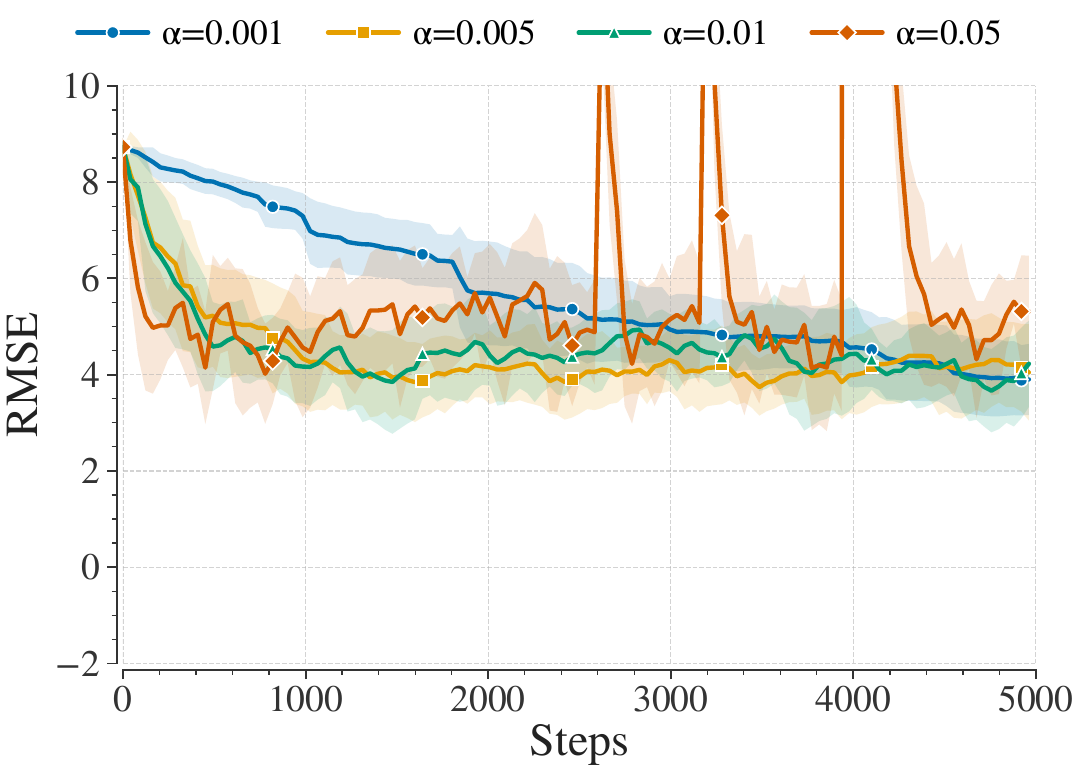}\caption{New two-state}\end{subfigure}
\caption{RETD learning-rate scan at the environment-specific regularization values used in the main comparisons. RETD remains stable across the swept stepsizes, indicating that the main-text choice \(\alpha=0.01\) is not a privileged setting.}
\label{fig:appendix_alpha_scan_fixed_c}
\end{figure}

\subsection{Structural reading of large \texorpdfstring{\(c\)}{c}}

The regularizer \(c\) controls how strongly the auxiliary centering variable acts on the value update. At \(c=0\) the method reduces to CETD and may inherit its indefinite coupled matrix; increasing \(c\) damps the auxiliary recursion and restores positive definiteness as analyzed in Section~\ref{sec:theory}; at very large \(c\) the auxiliary variable is forced close to zero and the update approaches the uncentered ETD update, so very large \(c\) can recover ETD's high-variance behavior and again yield inconsistent trajectories in the most aggressive off-policy regimes. This is the structural reason why the practical regime for \(c\) is intermediate, and it matches the empirical pattern reported in Figure~\ref{fig:appendix_c_scan_alpha_0p01}.
\FloatBarrier

\end{document}